\documentclass[preprint,5p,times,twocolumn]{elsarticle}

\usepackage{amssymb}
\usepackage{appendix}
\usepackage{amsmath}
\usepackage[misc]{ifsym}
\usepackage{multirow}
\usepackage{float}
\usepackage{multicol}
\usepackage{soul}
\usepackage{color, xcolor}
\usepackage[colorlinks=true, linkcolor=blue, citecolor=blue, urlcolor=blue]{hyperref}
\usepackage{subfigure}
\usepackage{algorithm}
\usepackage{algorithmic}
\usepackage{newfloat}
\usepackage{listings}
\usepackage{booktabs}
\usepackage{colortbl}
\usepackage{multirow}
\newtheorem{theorem}{Theorem}
\newtheorem{proposition}[theorem]{Proposition}

\newtheorem{corollary}[theorem]{Corollary}
\newtheorem{definition}[theorem]{Definition}

\newenvironment{proof}{{\noindent\it Proof.}\quad}{\hfill $\square$\par}

\journal{Arxiv}

\begin{document}

\begin{frontmatter}


\title{Introducing Diminutive Causal Structure into Graph Representation Learning}
\author{Hang Gao\fnref{label1,label2}}
\author{Peng Qiao\fnref{label1}}
\author{Yifan Jin\fnref{label1,label2}}
\author{Fengge Wu\fnref{label1,label2}}
\author{\Letter Jiangmeng Li\fnref{label1,label2}}
\ead{jiangmeng2019@iscas.ac.cn}
\author{Changwen Zheng\fnref{label1}}
\address[label1]{Science \& Technology on Integrated Information System Laboratory, Institute of Software, Chinese Academy of Sciences, Beijing, China. \fnref{label1}}
\address[label2]{University of Chinese Academy of Sciences, Beijing, China. \fnref{label2}}

\begin{abstract}
When engaging in end-to-end graph representation learning with Graph Neural Networks (GNNs), the intricate causal relationships and rules inherent in graph data pose a formidable challenge for the model in accurately capturing authentic data relationships. A proposed mitigating strategy involves the direct integration of rules or relationships corresponding to the graph data into the model. However, within the domain of graph representation learning, the inherent complexity of graph data obstructs the derivation of a comprehensive causal structure that encapsulates universal rules or relationships governing the entire dataset. Instead, only specialized diminutive causal structures, delineating specific causal relationships within constrained subsets of graph data, emerge as discernible. Motivated by empirical insights, it is observed that GNN models exhibit a tendency to converge towards such specialized causal structures during the training process. Consequently, we posit that the introduction of these specific causal structures is advantageous for the training of GNN models. Building upon this proposition, we introduce a novel method that enables GNN models to glean insights from these specialized diminutive causal structures, thereby enhancing overall performance. Our method specifically extracts causal knowledge from the model representation of these diminutive causal structures and incorporates interchange intervention to optimize the learning process. Theoretical analysis serves to corroborate the efficacy of our proposed method. Furthermore, empirical experiments consistently demonstrate significant performance improvements across diverse datasets.
\end{abstract}

\onecolumn


\end{frontmatter}

\section{Introduction} \label{sec:intro}

Graph representation learning is a specialized field that presents unique challenges compared to other representation learning tasks such as images, videos, or text. The studied graph data finds extensive practical applications in research, such as knowledge graphs \cite{yang2023lmkg, DBLP:journals/kbs/LiZMGWYYM22, DBLP:journals/kbs/LiYZJM23}, social networks \cite{DBLP:conf/www/WangHGLL23}, and molecular analysis \cite{DBLP:conf/nips/RongBXX0HH20, DBLP:conf/nips/WangLLLJ22, su2022molecular}. Graphs inherently encapsulate richer semantics as they encode intricate relationships and connections between entities \cite{DBLP:journals/tnn/WuPCLZY21}. Such property heightened semantic complexity often leads to increased intricacy in the learning process \cite{DBLP:conf/asl/PikhurkoV09}. Additionally, due to the nature of graph data, it is subject to more stringent structural constraints \cite{DBLP:journals/access/FaezOBR21}. However, when employing GNNs for graph representation learning, conventional GNNs model the existing associations in the data through an end-to-end approach, without delving into the exploration of the inherent causal relationships present in the data \cite{DBLP:journals/tai/00010YAWP021}. Therefore, when conducting end-to-end graph representation learning with GNNs, two primary challenges arise: effectively modeling the intricate relationships and rules within the complex graph data to capture genuine causal relationships, while simultaneously preventing the model from being perturbed by potential confounding factors \cite{DBLP:conf/iclr/0002Y21}.


Recently, methods including DIR \cite{DBLP:conf/iclr/WuWZ0C22} and RCGRL \cite{DBLP:conf/aaai/GaoLQSXZ023} have introduced causal learning mechanisms into GNN model training, such as interventions and instrumental variable, to enable the model to capture causal relationships better and eliminate confounding factors. Additionally, approaches such as RCL-OG \cite{DBLP:conf/aaai/YangYWW023} and ACD \cite{DBLP:conf/clear2/LoweMSW22} directly utilize causal discovery methods to construct causal graphs from data. There is also some methods that focus on utilizing causal learning to enhance the interpretability of GNN models \cite{DBLP:conf/iclr/Liu0X23, DBLP:conf/aistats/AgarwalZL22, DBLP:conf/iclr/CucalaGKM22}. These methods have demonstrated significant effectiveness and contributed to the optimization of graph learning tasks. However, it is crucial to note that none of these methods can guarantee that the model accurately learns causal relationships. The incorporation of causal learning methods does not ensure that the relationships modeled by the model are genuinely causal. Simultaneously, methods for constructing causal graphs or boosting model interpretability struggle to effectively leverage discovered causal knowledge to guide GNN models, thus limiting their ability to enhance GNN models' capability to model causal relationships. Therefore, we cannot help but wonder: If we possess prior knowledge about specific causal relationships, can we inject this knowledge into GNN models and ensure that the model can reliably identify and model these relationships? Furthermore, can we ensure that such injection of prior knowledge is applicable and helpful to a wide range of scenarios?

One approach to addressing this challenge is to directly inject the rules or relationships that correspond to the graph data into the model. For typical machine learning tasks, such as image recognition and text analysis, we possess well-founded insights about causal structures that we can express symbolically. These insights range from common-sense intuitions about how the world operates to advanced scientific knowledge. Recent research shows that introducing such causal structures into neural network models can significantly enhance their performance \cite{geiger2022inducing}. However, in the context of graph representation learning, achieving this directly is challenging. 
Due to the complexity of graph data, it is difficult to identify a comprehensive causal structure that can describe the general rules or relationships governing graph data as a whole. Instead, we often find ourselves dealing with what we term `diminutive causal structures', which can only be used to describe the rules that apply to specific, limited subsets of graph data (Please refer to Definition \ref{df:ck} for a precise definition of `diminutive causal structure'). Therefore, there are several fundamental questions that need to be answered here: What is the relationship between graph representation learning models and diminutive causal structures? Does the introduction of such diminutive causal structures still contribute to improving the effectiveness of graph representation learning? 

\begin{figure*} 
\subfigure[Visualization of a conventional GNN approach's learning degree of the diminutive causal structures during optimization on benchmark datasets. We specifically measure the cosine similarity between the output representations of the causal model and the GNN model.] { \label{fig:mtv1}
\includegraphics[width=0.7\columnwidth]{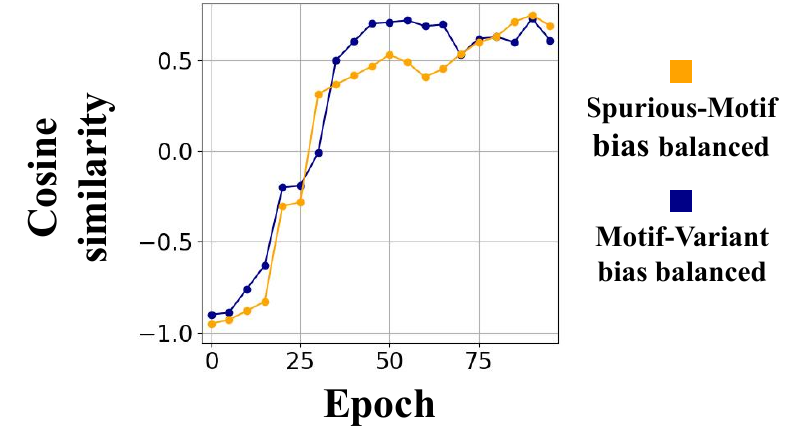}
}\hspace{3mm}
\subfigure[Comparisons between the backbone GNNs (Local Extremum GNN \cite{DBLP:conf/aaai/RanjanST20} and ARMA \cite{DBLP:conf/aaai/0001RFHLRG19}) and the proposed method on various datasets. The left plot demonstrates the KL divergence between the features derived by the model capable of representing diminutive causal structures and the candidate GNN model. The right plot collects the classification performance of the compared methods. The empirical observations jointly support that our proposed method can better learn the diminutive causal structures and consistently outperform the baseline.] {\label{fig:mtv2}
\includegraphics[width=1.3\columnwidth]{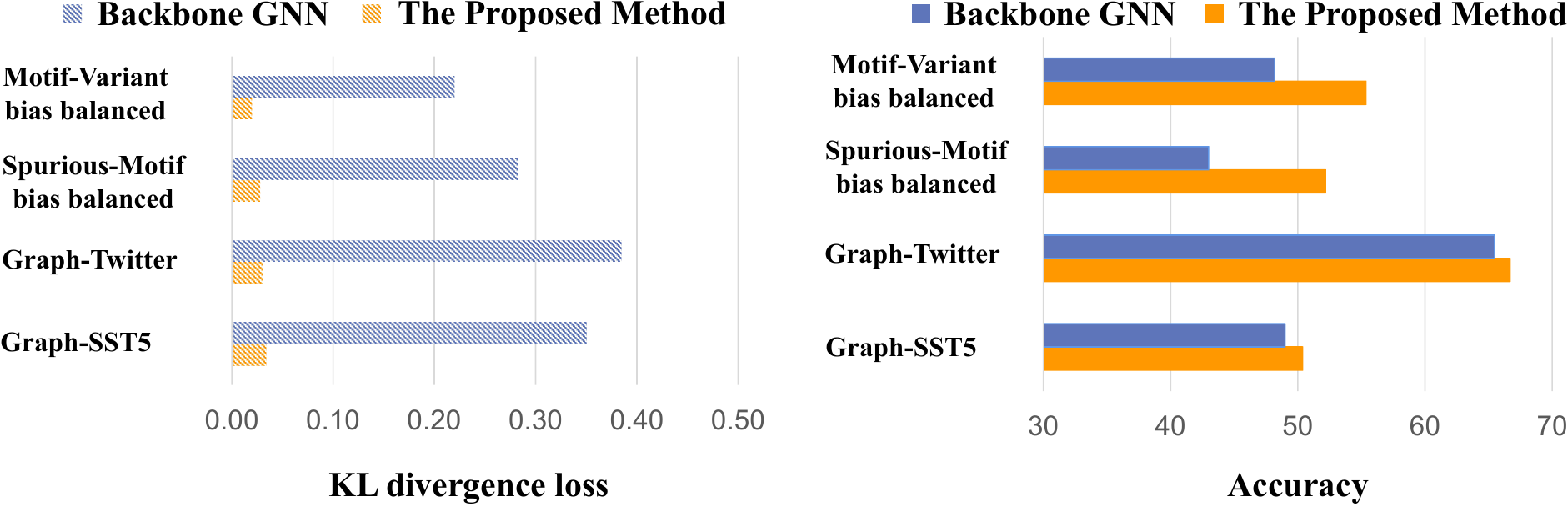}
}
\vskip -0.1in
\caption{ The motivating explorations. }
\vskip -0.2in
\label{fig}
\end{figure*} 

To acquire the answers, we intuitively introduce the experiments to investigate the relationship between GNN models and diminutive causal structures. As demonstrated in Figure \ref{fig:mtv1}, GNNs tend to converge towards models capable of representing diminutive causal structures for specific graph data as training progresses. Based on this observation, we propose that as training progresses and GNN acquires more information about the training dataset, certain internal structures within the GNN model gradually approach diminutive causal structures. This phenomenon suggests that such causal structures benefit the task, leading the GNN to gradually approximate similar structures during end-to-end training. However, it is evident that the GNN fails to achieve complete consistency with diminutive causal structures. Hence, we hypothesize that actively introducing such causal structures can further optimize the GNN and yield improved performance. To facilitate further exploration, we conduct experiments on various domains in Figure \ref{fig:mtv2}. Based on the empirical results, it can be observed that there is a general boost of model performances of GNNs introducing diminutive causal structures to a certain extent across different tasks. The observation leads to an empirical conclusion: \textit{introducing diminutive causal structure into graph representation learning practically improves the model performance}.


In light of such empirical finding, we propose the \textit{\underline{D}iminutive \underline{C}ausal \underline{S}tructure guided \underline{G}raph Representation \underline{L}earning}, dubbed \textit{DCSGL}, to improve the model to learn the diminutive causal structure that is related to the graph representation learning tasks. DCSGL achieves the aforementioned objectives by enabling the GNN model to learn knowledge from higher-level causal models constructed with diminutive causal structures. We also conduct a theoretical analysis of DCSGL based on the Structural Causal Model (SCM) \cite{DBLP:journals/ijon/Shanmugam01, pearl2009causal, glymour2016causal}, demonstrating its effectiveness. Moreover, we further incorporate the concept of interchange intervention \cite{geiger2022inducing}, enabling DCSGL to facilitate more comprehensive learning. The sufficient comparisons on various datasets, including the crafted and real-world datasets, support the consistent effectiveness of DCSGL. 

\textbf{Contributions}:

\begin{itemize}

    \item We discover an intriguing observation, which leads to an empirical conclusion: introducing diminutive causal structure into GNNs improves the model performance. Additionally, we undertake a thorough investigation, employing comprehensive theoretical analyses and rigorous proofs, to substantiate and elucidate this conclusion.

    \item We introduce a novel DCSGL method, which enables the GNN model to learn diminutive causal structure that is associated with the graph representation learning task, thereby refining its performance.

    \item Extensive empirical evaluations on various datasets, including crafted and benchmark datasets, demonstrate the effectiveness of the proposed DCSGL.
    
\end{itemize}

The outline of the article includes a review of related work in Section 2, covering research on causal learning and GNNs, as well as the application of causal learning methods to GNNs. Section 3 outlines the methodology, introducing foundational concepts of causal structures and Diminutive Causal Structures, along with interventional exchange methods. Theoretical analysis in Section 4 includes the introduction of a structural causal model and a detailed examination of the proposed approach. Finally, Section 5 validates the approach through empirical studies, leveraging topological and semantic domain knowledge for comparison with other methods and further experimental analysis.

\section{Related Works}


\subsection{Causal Learning} 
Causal learning employs statistical causal inference techniques \cite{glymour2016causal} to uncover the causal relationships among observable variables. Integrating causal learning into deep learning algorithms is a growing area of interest. Researchers are exploring how causal models can enhance the performance, interpretability, and fairness of deep learning models \cite{krueger2021out,DBLP:conf/icml/ZhouLZZ22,DBLP:conf/cvpr/LinDWZ22,jin2022supporting,li2022supporting,li2022metamask}. This includes developing methods to incorporate causal assumptions \cite{geiger2022inducing}, leveraging causal structures in deep learning architectures \cite{DBLP:conf/nips/0001ZDZ22}, and incorporating causal explanations into model predictions \cite{DBLP:conf/aaai/LiCM22}. 

However, traditional causal inference, often reliant on direct interventions in the studied system \cite{DBLP:conf/aaai/ShpitserP06}, encounters challenges when employed in deep learning models. The causal learning methods used in deep learning models often struggle to implement changes in experimental scenarios corresponding to the dataset. Consequently, while these methods have demonstrated promising results, they cannot guarantee the definitive learning of causal relationships.


\subsection{Graph Neural Networks}
GNNs have introduced the mechanism of neural networks into the realm of graph representation learning. Since the inception of the first GNN model, multiple variants have been proposed, including GCN \cite{kipf2016semi}, GAT \cite{velivckovic2017graph}, GIN \cite{xu2018powerful}, Local Extremum GNN \cite{DBLP:conf/aaai/RanjanST20}, GSAT \cite{miao2022interpretable} among others. These models perform graph representation learning by propagating information between adjacent nodes. In addition, to improve performance and adapt to a broader range of downstream tasks, paradigms commonly employed by other types of deep learning architectures have been applied to GNNs, such as unsupervised learning \cite{you2020does, xu2021infogcl, DBLP:conf/aaai/CoorayC22, SUN2021107564}, semi-supervised learning \cite{baranwal2021graph, DBLP:conf/nips/YueZZL22}, meta learning \cite{DBLP:conf/ijcai/GaoLQS0Z22, DBLP:conf/kdd/GuoLZZZX22}, and reinforcement learning \cite{jiang2019graph, DBLP:conf/kdd/0001GZMTLK22}. The aforementioned GNNs, akin to other types of neural networks, are trained using an end-to-end approach, fundamentally modeling the probabilistic relationships between the data. However, this manner of relationship modeling hinders the accurate representation of causal relationships, particularly when dealing with intricate and semantically rich graph data.

\subsection{Application of Causal Learning within GNNs} 


Despite the achievements of GNNs, these approaches encounter difficulties in explicitly uncovering causal relationships. To overcome this limitation, researchers have integrated causal inference techniques into GNN-based graph representation models \cite{DBLP:journals/sensors/HoangJYYJL23, prado2022survey}. Some focused on utilizing causal learning to enhance the interpretability of GNN models \cite{DBLP:conf/iclr/Liu0X23, DBLP:conf/aistats/AgarwalZL22, DBLP:conf/iclr/CucalaGKM22, DBLP:conf/aaai/LiDSQHC23, DBLP:conf/icml/WuLJJRNL23}, e.g., GLGExplainer \cite{DBLP:conf/iclr/AzzolinLBLP23} tries to represent the causal structure within the model more understandably. Some methods attempt to directly mine causal structures from the data \cite{DBLP:conf/aaai/YangYWW023,DBLP:conf/clear2/LoweMSW22}. Other methods facilitate the incorporation of causality through data intervention \cite{DBLP:conf/iclr/WuWZ0C22,DBLP:conf/nips/0002ZB00XL0C22, DBLP:conf/aaai/GaoLQSXZ023}. Furthermore, some works have successfully combined GNNs with causal learning for time-series data, achieving promising results \cite{duan2022multivariate, DBLP:journals/prl/WangDHXFR22,gao2023rethinking}. As a representative approach, DIR \cite{DBLP:conf/iclr/WuWZ0C22} proposes to employ adaptive graph neural networks to identify distinct data within a graph, and then guide causal interventions to introduce causality into the model to boot the performance of models within Out-Of-Distribution (OOD) tasks.

However, the aforementioned methods, while achieving satisfactory results, still exhibit certain limitations. Approaches aiming to enhance model interpretability and conduct causal model mining have not effectively reinforced GNN models using the acquired causal relationships. Similarly, methods that leverage causal learning for GNN model improvement encounter challenges in accurately learning causal relationships due to the intricate nature of graph representation learning. Theoretically, the prerequisites of causal reasoning methods are challenging to meet in the context of GNNs, given the complex nature of data processing and the inherent nature of end-to-end training. Therefore, our approach endeavors to take a different route by directly injecting easily obtainable diminutive causal structures into the model to address this issue.

\section{Methodology}

\begin{figure*}[ht]
	\centering
    \includegraphics[width=1\textwidth]{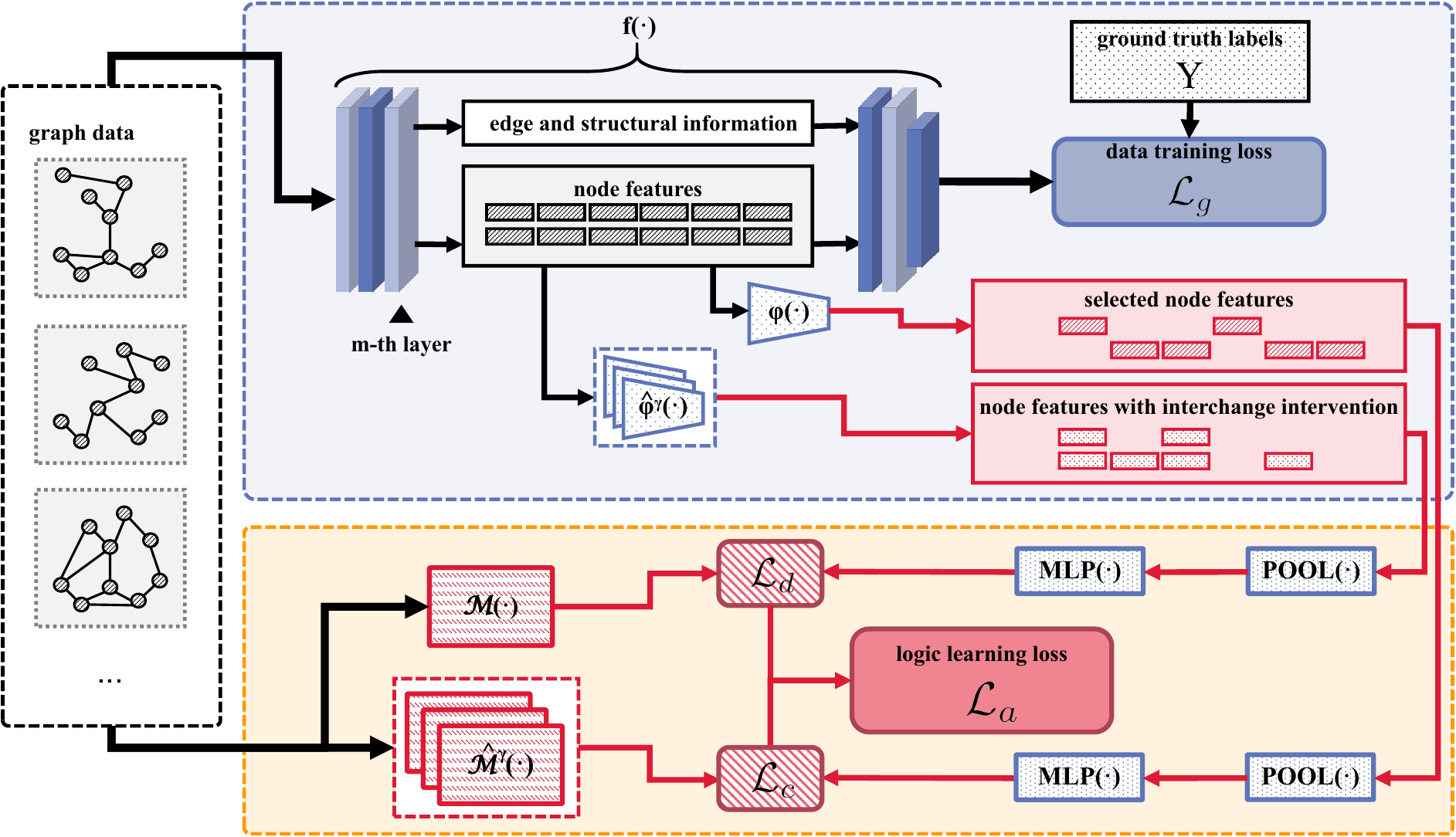} \\
	\caption{The overall framework of DCSGL.}
    \label{fig:framework}
\end{figure*}

In this section, we introduce our methodology, DCSGL, providing a detailed exposition on our proposed concept of `diminutive causal structures' and explaining how we integrate them into our model. We will also outline the enhancements made to the entire process to achieve optimal results. For clarification, we summarizes the main notations within Table \ref{tab:gn}.

\begin{table}[h] \scriptsize
	\centering
 	\caption{Glossary of notations.}
    \vskip 0.1in
	\begin{tabular}{c|c}
		\hline\rule{0pt}{8pt}
		Notation      &  Description \\
		\hline\rule{0pt}{8pt}
		$G$ & Graph sample. \\\rule{0pt}{8pt}
		$X$ & The causal part within the graph sample. \\\rule{0pt}{8pt}
        $Y$ & Label \\\rule{0pt}{8pt}
        $C$ & The confounder within the graph sample. \\\rule{0pt}{8pt}
        $\widetilde{S}$ & Factors influencing $X$ that can be explicitly delineated.  \\\rule{0pt}{8pt}
        $\mathcal{M}(\cdot)$ & Concrete model that represent certain diminutive causal structure. \\\rule{0pt}{8pt}
        $\phi(\cdot)$ & The node selection procedure. \\\rule{0pt}{8pt}
        $\widetilde{T}$ & The output prediction of $\mathcal{M}(\cdot)$. \\\rule{0pt}{8pt}
        $q(\cdot)$ & $q(\cdot)$ represents the probability density function that $\widetilde{T}$ follows. \\\rule{0pt}{8pt}
        $t$ & The value. \\\rule{0pt}{8pt}
        $\mathcal{f}(\cdot)$ & GNN model. \\\rule{0pt}{8pt}
        $T$ & The prediction of $\widetilde{T}$ given based on $\mathcal{f}(\cdot)$. \\\rule{0pt}{8pt}
        $p(\cdot)$ & $p(\cdot)$ represents the probability density function that $T$ follows. \\\rule{0pt}{8pt}
        $\mathcal{T}$ & The value space of $\widetilde{T}$ and $\widetilde{T}$. \\\rule{0pt}{8pt}
        $KL(\cdot)$ & The KL-divergence. \\\rule{0pt}{8pt}
        $\hat{\mathcal{M}}^{\gamma}(\cdot)$ & $\mathcal{M}(\cdot)$ after $\gamma$-th  interchange intervention \\\rule{0pt}{8pt}
        $\hat{\phi}^{\gamma}(\cdot)$ & $\phi(\cdot)$ after $\gamma$-th interchange intervention \\\rule{0pt}{8pt}
        $\mathcal{D}_{KL}$ & Function that calculates the KL divergence via a discrete manner. \\\rule{0pt}{8pt}
        $G_{i}$ & The $i$-th graph sample. \\\rule{0pt}{8pt}
        $N$ & Total number of the graph samples. \\\rule{0pt}{8pt}
        $K$ & Total number of interchange intervention. \\\rule{0pt}{8pt}
        $f^{m}(\cdot)$ & The output of the $m$-th layer of $f(\cdot)$. \\\rule{0pt}{8pt}
        $POOL(\cdot)$ & Average pooling. \\\rule{0pt}{8pt}
        $MLP(\cdot)$ & Multi-layer perception. \\\rule{0pt}{8pt}
        $I(;)$ & Mutural information.\\\rule{0pt}{8pt}
        $S^{*}$ & Factors influencing $X$ that can't be explicitly delineated.\\\rule{0pt}{8pt}
        $R$ & The model's representation of $G$. \\
		\hline
	\end{tabular}
	\label{tab:gn}
\end{table}

\subsection{Preliminary}
We begin by reviewing some concepts related to causal learning, as they will be employed in the subsequent narrative.

\subsubsection{Causal Structure}
Causal structure refers to the underlying arrangement or configuration of causal relationships among variables or events in a system \cite{DBLP:journals/jmlr/Pearl10}. It represents the way in which different elements influence one another and the directionality of these influences. Understanding the causal structure of a system is crucial in causal inference and learning, as it helps uncover the relationships between variables and enables predictions or interventions based on a deeper comprehension of the cause-and-effect mechanisms at play. We will leverage the concept of causal structure, coupled with graph representation learning, to analyze which types of causal Structures are more readily attainable and expressible. Additionally, we will elucidate how to incorporate them into our GNN model.

\subsubsection{Interchange Intervention}
\label{sec:intinv}
For a model $\mathcal{M}$ that is used to process two different input samples, $X_{a}$ and $X_{b}$, the interchange intervention \cite{geiger2022inducing} conducted on a model $\mathcal{M}$ can be viewed as providing the output of the model $\mathcal{M}$ for the input $X_{a}$, except the variables $V$ are set to the values they would have if $X_{b}$ were the input. Interactive interventions enable improved training of the internal logical structure of a model without the need for additional sample augmentation. We employ interchange intervention to artificially modify the feature representations within the training model, facilitating a better alignment of the model with diminutive causal structures.

\subsection{Diminutive Causal Structure}

\begin{figure}[h]
	\centering
    \includegraphics[width=0.35\textwidth]{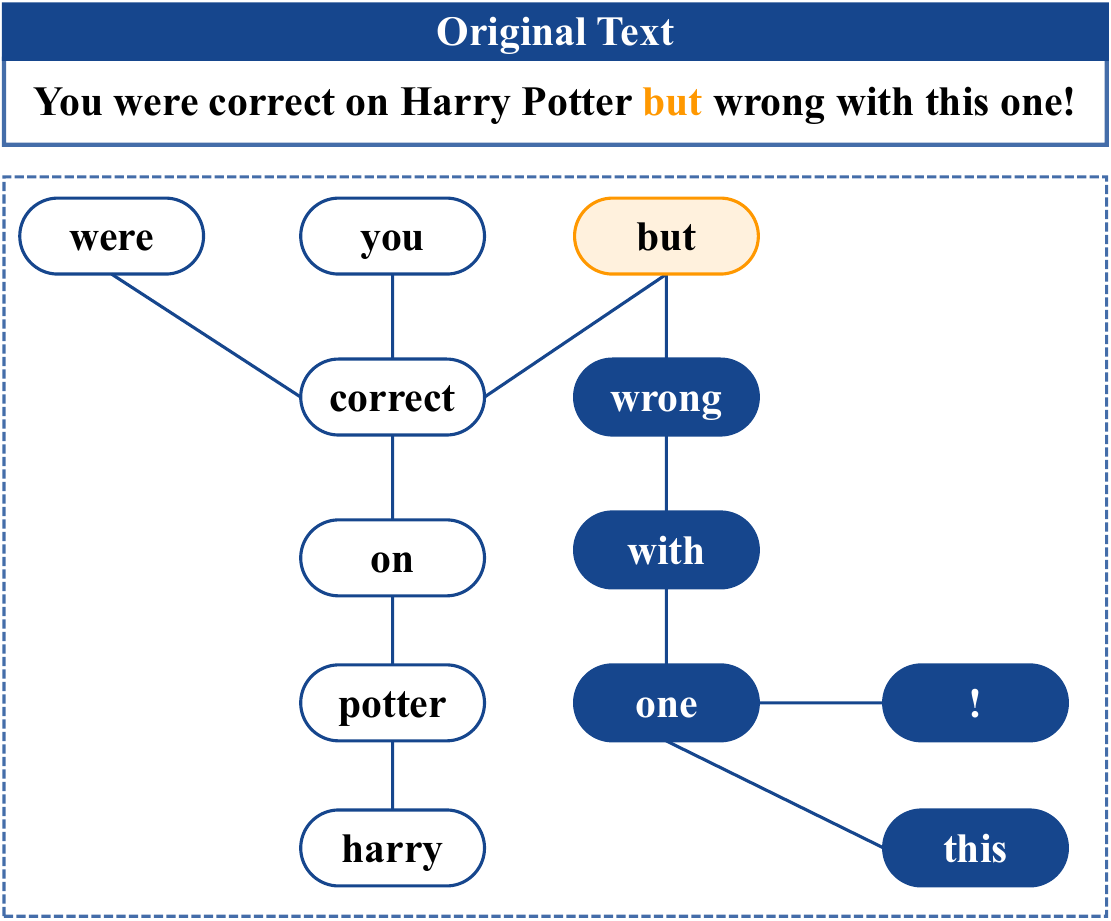} \\
	\caption{Example graph sample.}
    \label{fig:example}
\end{figure}



In order to convey the information more clearly, we provide the precise definition of diminutive causal structure. For graph $G$, we denote the graph data possesses causal relationships with the ground truth labels as causal part $X$ and the data unrelated to the labels as confounder $C$. In graph representation learning, accurately modeling the causal relationship between $X$ and $Y$ proves challenging, involving intricate task-specific knowledge that is complex and difficult to analyze and acquire. However, more generalized foundational knowledge about $X$ can be acquired. Specifically, based on domain-specific knowledge, we can analyze the factors $\widetilde{S}$ that influence the composition of $X$. The causal structure between $\widetilde{S}$ and $X$ is then referred to as the diminutive causal structure. We provide its formal definition below:

\begin{definition}
\label{df:ck}
(Diminutive Causal Structure). Diminutive causal structure denotes a causal structure that represents the causal relationship between a subset of data in $X$ and the factors $\widetilde{S}$. $\widetilde{S}$ be the causes of $X$. Furthermore, there is no direct causal effect between $\widetilde{S}$ and $Y$. 
\end{definition}

We will illustrate the diminutive causal structure based on a specific example. In Figure \ref{fig:example}, we present a graph sample from the Graph-Twitter dataset \cite{yuan2020explainability}. The Graph-Twitter dataset consists of a series of syntactic dependency tree graphs extracted from a collection of movie comments. The task of the dataset is to determine, based on these syntactic dependency tree graphs, whether the corresponding comments contain positive or negative evaluations. 

Specifically regarding the sample illustrated in Figure \ref{fig:example}, we have provided its corresponding raw textual data, namely, `You were correct on Harry Potter but wrong with this one!' In this sentence, the conjunction `but' is present, and based on linguistic knowledge, we can determine that there is a semantic contrast between the content before and after `but.' In the graph below, we use white nodes to represent the content before `but' and deep blue to represent the content after `but.' 

At this point, the presence of `but' in the sentence can be considered as cause factor $\widetilde{S}$, and such $\widetilde{S}$ leads to a semantic shift between the content represented by the white nodes and that represented by the deep blue nodes in the graph. Furthermore, the knowledge related to `but' is associated with the causal information $X$ in the graph data $G$. As the Graph-Twitter dataset is a semantic sentiment analysis dataset, domain-specific knowledge, such as `but' indicating a semantic contrast, is relevant to the task. We can consider that conjunction `but' leads to a semantic contrast as a diminutive causal structure because it reveals a causal relationship between $X$ and $\widetilde{S}$. Meanwhile, `but' is not directly related to the labels; the presence or absence of `but' cannot be used to determine whether the entire graph corresponds to a positive or negative evaluation.


\subsection{Guided Learning}

Next, we incorporate the diminutive causal structure into our model to enhance its performance. To achieve this, we represent the diminutive causal structure as a concrete model, denoted as $\mathcal{M}(\cdot)$. $\mathcal{M}(\cdot)$ output the prediction $\widetilde{T}$ about $\widetilde{S}$ based on graph data $G$. For instance, using the example above, we can define $\mathcal{M}(\cdot)$'s output $\widetilde{T}$ as whether the conjunction `but' exists in the graph. In Section \ref{sec:mds1} and \ref{sec:mds2}, we provide more detailed descriptions of the specific designs of $\mathcal{M}(\cdot)$ for different diminutive causal structures.

With $\mathcal{M}(\cdot)$, our objective is to enable our GNN model, denoted as $f(\cdot)$, to possess an internal causal structure that realizes $\mathcal{M}(\cdot)$, so as to learn the diminutive causal structure. In simple terms, we believe incorporating the diminutive causal structure within $\mathcal{M}(\cdot)$ into our model is advantageous for training. This is akin to introducing domain-specific knowledge to the model. We also leverage Theorem \ref{th:c} to theoretically substantiate the effectiveness of this approach.

To achieve the objective of introducing $\mathcal{M}(\cdot)$, we enforce $f(\cdot)$ to output a prediction of $\widetilde{S}$. We denote such prediction as $T$, $T \sim q(t)$. $t$ is the value. $q(t)$ is the corresponding probability density function of $T$. $\widetilde{T}$ and $T$ share the same value space $\mathcal{T}$. Furthermore, we assume $\widetilde{T} \sim p(t)$. $p(t)$ is the corresponding probability density function of $\widetilde{T}$. Intuitively, if given $\widetilde{S}$, the output distributions of $f(\cdot)$ and $\mathcal{M}(\cdot)$ are the same, i.e., the following equation holds: 
\begin{gather}
p(t|\widetilde{s})=q(t|\widetilde{s}), t \in \mathcal{T}, \widetilde{s} \in \mathcal{\widetilde{S}}, 
\end{gather} 
then $f(\cdot)$ has learned the knowledge that $\mathcal{M}(\cdot)$ possesses. Theorem \ref{th:ori} gives a theoretical justification for such intuition. We utilize the KL divergence to assess the distance between $p(t|\widetilde{s})$ and $q(t|\widetilde{s})$, and formulate our training objective as follows: 
\begin{gather}
\mathcal{L}_{o} = KL\Big(p(t|\widetilde{s}) || q(t|\widetilde{s})\Big). 
\label{eq:Lo} 
\end{gather} 

In practical, we could let $\mathcal{M}(\cdot)$ and $f(\cdot)$ directly predict the probability values of $p(t|\widetilde{s})$ and $q(t|\widetilde{s})$. For training set $\{ G_{i} \}^{N}_{i=1}$, we propose the following loss function: 
\begin{gather} 
\mathcal{L}_{c} = \sum_{i=1}^{N} \mathcal{D}_{KL}\Big(\mathcal{M}(G_i), f(G_i)\Big), 
\label{eq:Lc2} 
\end{gather} 
where $\mathcal{D}_{KL}(\cdot)$ calculates the KL divergence of two input distributions via a discrete manner. We have provided a justification for the rationality of this loss function through Corollary \ref{corollary:Lc2}. 

With Equation \ref{eq:Lc2}, we are able to align the output of function $f(\cdot)$ with that of $\mathcal{M}(\cdot)$. By doing so, we effectively induce convergence in the internal logic of $f(\cdot)$ towards that of $\mathcal{M}(\cdot)$. However, for downstream tasks, our objective is not to make $f(\cdot)$ identical to $\mathcal{M}(\cdot)$. What we aspire to achieve is the incorporation of the causal structure represented by $\mathcal{M}(\cdot)$ within $f(\cdot)$, while allowing $f(\cdot)$ to acquire necessary knowledge from generic dataset-based training. To accomplish this, we constrain the scope of the impact of the loss function $\mathcal{L}_{c}$. Specifically, we opt to select specific node features from the $m$-th layer of $f(\cdot)$, with $m$ as a hyperparameter. In essence, we choose to extract node features from the $m$-th layer of $f(\cdot)$ as a means to ensure that $f(\cdot)$ captures the causal structure represented by $\mathcal{M}(\cdot)$, while retaining the ability to learn essential knowledge through general dataset-based training. Furthermore, we only select the graph data that may be related to $S$. For instance, for the data corresponding to Figure \ref{fig:example}, we only choose the features of the lexical nodes corresponding to the occurrence of `but' in the sentence. We denote the selection procedure as $\phi(\cdot)$. As described in the example above, the framework of $\phi(\cdot)$ is closely tied to the chosen diminutive causal structure. Therefore, we provide a detailed explanation of the concrete framework of $\phi(\cdot)$ in the specific exposition of $\mathcal{M}(\cdot)$ in sections \ref{sec:mds1} and \ref{sec:mds2}. Next, we conduct average pooling \cite{kipf2016semi} on the selected node features, the project them into probability predictions using a Multi-layer Perception (MLP) \cite{DBLP:conf/icaisc/Grum23}. Therefore, we can formulated our altered loss fuction as follows:
\begin{gather} 
\mathcal{L}_{c} = \sum_{i=1}^{N} \mathcal{D}_{KL}\Big(\mathcal{M}(G_i), MLP \circ POOL \circ \phi \circ f^{(m)}(G_i) \Big), 
\label{eq:Lc}
\end{gather} 
where $MLP(\cdot)$ denotes the Multi-layer Perception, $POOL(\cdot)$ denotes the average pooling operation, $\circ$ represents the composition of functions, $f^{(m)}(G_i)$ denote the output of the $m$-th layer of $f(\cdot)$.

\subsection{Enhancement with Interchange Interventions}
\label{sec:intercinterv}

The aforementioned design delineates a viable approach to facilitate the GNN model in acquiring knowledge regarding the diminutive causal structure, enabling $f(\cdot)$ to assimilate precise causal structural information encapsulated in $\mathcal{M}(\cdot)$. However, it is noteworthy that the dataset employed may not comprehensively cover graph data related to $\widetilde{S}$; in fact, only a minute fraction may be directly associated with $\widetilde{S}$. For example, as previously mentioned, the incidence rate of the word 'but' in the sample extracted from the Graph-Twitter dataset is merely 5\%. Such limited data may prove inadequate for robust model training. To ensure that the model $f(\cdot)$ can effectively capture the desired knowledge, we leverage interchange intervention \cite{geiger2022inducing, DBLP:conf/emnlp/GeigerCKP19}, a method rooted in causal theory , to facilitate and augment the learning process. A comprehensive introduction to interchange intervention is provided in Section \ref{sec:intinv}.

We conduct interchange intervention on the GNN model by modifying $\phi(\cdot)$, which selects the node features for training. By changing $\phi(\cdot)$, we effectively altered the original input values of the subsequent neural network, akin to what an interchange intervention accomplishes. Subsequently, given our explicit understanding of the intervention we performed, we can instruct model $\mathcal{M}(\cdot)$ to provide accurate altered predictions for $\widetilde{T}$ given certain interchange intervention, thereby augmenting the quantity of data available for training. We continue to elucidate using the example in Figure \ref{fig:example}. In Figure \ref{fig:example}, the nodes within the graph are input into the model for training. We have the capability to selectively remove specific nodes from such sentences. For instance, we can eliminate 'but' and the subsequent content, thereby altering the semantic meaning of the sentence. Clearly, in this case, the word 'but' no longer exists in the sentence, along with its associated semantic shift. We can instruct model $\mathcal{M}(\cdot)$ to predict the probability of the existence of `but' as being zero under these circumstances. 

Our desideratum is that the output of the $f(\cdot)$ and $\mathcal{M}(\cdot)$ are the same under identical interchange interventions, so as to align the two models. We denote the $\phi(\cdot)$ and $\mathcal{M}(\cdot)$ after interchange intervention as $\hat{\phi}^{\gamma}(\cdot)$ and $\hat{\mathcal{M}}^{\gamma}(\cdot)$ correspondingly, where $\gamma$ represents the index of different interchange interventions. The specific design details of $\hat{\mathcal{M}}^{\gamma}(\cdot)$ and $\hat{\phi}^{\gamma}(\cdot)$ are expounded in Section \ref{sec:mds1} and \ref{sec:mds2}. Then, we propose the following loss function:
\begin{gather} 
\mathcal{L}_{d} = \sum_{i=1}^{N} \sum_{\gamma=1}^{K} \mathcal{D}_{KL}\Big(\hat{\mathcal{M}}^{\gamma}(G_i), MLP \circ POOL \circ \hat{\phi}^{\gamma} \circ f^{(m)}(G_i) \Big), 
\end{gather} 
where $K$ denotes the total number of interchange interventions.

\subsection{Training Procedure}
The overall framework is illustrated in Figure \ref{fig:framework}. For learning diminutive causal structure, we define our training objective with both $\mathcal{L}_{c}$ and $\mathcal{L}_{d}$:
\begin{gather}
\mathcal{L}_{a} = \mathcal{L}_{c} + \lambda \mathcal{L}_{d},
\label{eq:L2}
\end{gather}
$\mathcal{L}_{a}$ denotes the total loss for diminutive causal structure learning, $\lambda$ is a hyperparameter that balances the influence of learning under interchange intervention. Besides diminutive causal structure, our GNN model is also trained with conventional labeled data. The training object of which can be formulated as:
\begin{gather}
	 \mathcal{L}_{g} = \sum_{i=1}^{N} \mathcal{H}\Big(  f(G_{i}), Y_{i}\Big),
	\label{eq:lg0}
\end{gather}
where $\mathcal{H}(\cdot)$ calculates the cross entropy loss, and $f^{H}(\cdot)$ denotes the projection head for label prediction. We update the model alternatively with $\mathcal{L}_{g}$ and $\mathcal{L}_{a}$. Then $f(\cdot)$ will be utilized for the performance test.

\section{Theoretical Analysis}
\label{s:tac}

\subsection{Causal Modeling with Structural Causal Model} 
 To embark on our investigation, it is imperative to establish a sound modeling of our problem scenario. We employ Structural Causal Models (SCM) to fulfill this task. SCM \cite{DBLP:journals/ijon/Shanmugam01} is a framework utilized to describe the manner in which nature assigns values to variables of interest. Formally, an SCM consists of variables and the relationships between them. Each SCM can be represented as a graphical model consisting of a set of nodes representing the variables, and a set of edges between the nodes representing the relationships. 

\begin{figure}
    \centering
    \includegraphics[width=0.15\textwidth]{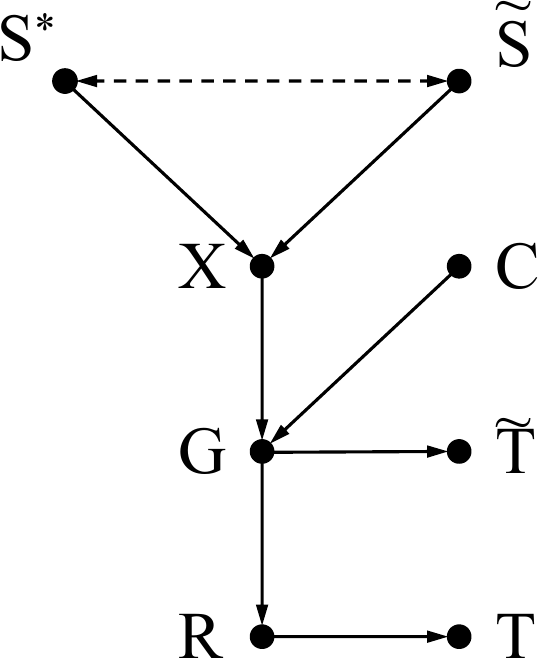}
    \caption{Graphical representation of the SCM.}
    \vskip -0.2in
    \label{fig:SCM}
\end{figure}

We formalize the candidate problem by using an SCM illustrated in Figure \ref{fig:SCM}. Among the SCM, $\widetilde{S}$, as defined before, refers to those factors that influence the structure of graph data, and their causal relationships with graph data can be explicitly delineated. $S^{*}$ denotes the rest factors. $\widetilde{S}$ and $S^{*}$ together decide the formulation of $X$, which represents the data within the graph $G$ that holds causal relationship with the ground truth labels. $C$ represents any confounding factors present within the data, $G$ represents the graph data itself, and $R$ represents the output representation of $f(\cdot)$. $\widetilde{T}$ and $T$ are the outputs of the high-level causal model and the GNN. The links in Figure \ref{fig:SCM} are as follows: 

\begin{itemize}
\item $S^{*} \to X, \widetilde{S} \to X.$ The form of $X$ is desicde with $S^{*}$ and  $\widetilde{S}$.
\item $X \to G \gets C.$ The graph data $G$ consists of two parts: $X$ and confounder $C$.
\item $G \to R.$ $f(\cdot)$ encodes $G$ into representation $R$ that consists of the node representations of different layers and the graph representation.
\item $S^{*} \leftarrow - \rightarrow \widetilde{S}.$ The bidirectional arrow with a dashed line indicates that the causal relationship between the two cannot be confirmed.
\item $G \to \widetilde{T} , R \to T.$ $\widetilde{T}$ and $T$ are variables that calculated based on $G$ and $R$.
\end{itemize}

Furthermore, we propose the following proposition:
\begin{proposition}
The SCM illustrated in Figure \ref{fig:SCM} can represent scenarios encountered when applying our model to general graph learning tasks.
\end{proposition}

\begin{proof}
To establish the validity of the proposed SCM in Figure \ref{fig:SCM}, we employ the IC algorithm \cite{DBLP:journals/ijon/Shanmugam01} to construct the SCM from scratch and provide the detailed construction process. IC algorithm is a method for identifying causal relationships from the observed data. Please refer to Chapter 2 of \cite{DBLP:journals/ijon/Shanmugam01} for the details of the IC algorithm. 

The input of the IC algorithm is a set of variables and their distributions. The output is a pattern that represents the underlying causal relationships, which can be a structural causal model. The IC algorithm can be divided into three steps:

\paragraph{Step 1} For each pair of variables $a$ and $b$ in $V$, the IC algorithm searches for a set $S_{ab}$ s.t. the conditional independence relationship $(a \perp b | S_{ab})$ holds. In other words, $a$ and $b$ should be independent given $S_{ab}$. The algorithm constructs an undirected graph $G$ with vertices corresponding to variables in $V$. A pair of vertices $a$ and $b$ are connected with an undirected edge in $G$ if and only if no set $S_{ab}$ can be found that satisfies the conditional independence relationship $(a \perp b | S_{ab})$.

\paragraph{Step 2} For each pair of non-adjacent variables $a$ and $b$ that share a common neighbor $c$, check the existence of $c \in S_{ab}$. If such a $c$ exists, proceed to the next pair; if not, add directed edges from $a$ to $c$ and from $b$ to $c$.

\paragraph{Step 3} In the partially directed graph results, orient as many of the undirected edges as possible subject to the following two conditions: (i) any alternative orientation of an undirected edge would result in a new y-structure, and (ii) any alternative orientation of an undirected edge would result in a directed cycle. 

Accordingly, we employ IC algorithm to provide a step-by-step procedure for constructing our structural causal model.

\begin{figure*}[ht]
	\centering
	\subfigure[]{
		\begin{minipage}{0.15\textwidth} 
			\includegraphics[width=\textwidth]{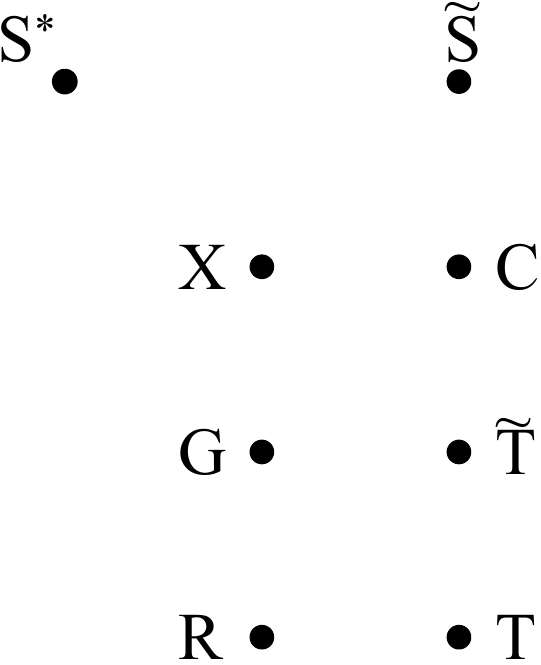} \\
                \label{IC-a}
		\end{minipage}
    }\hspace{10mm}
	\subfigure[]{
		\begin{minipage}{0.15\textwidth} 
			\includegraphics[width=\textwidth]{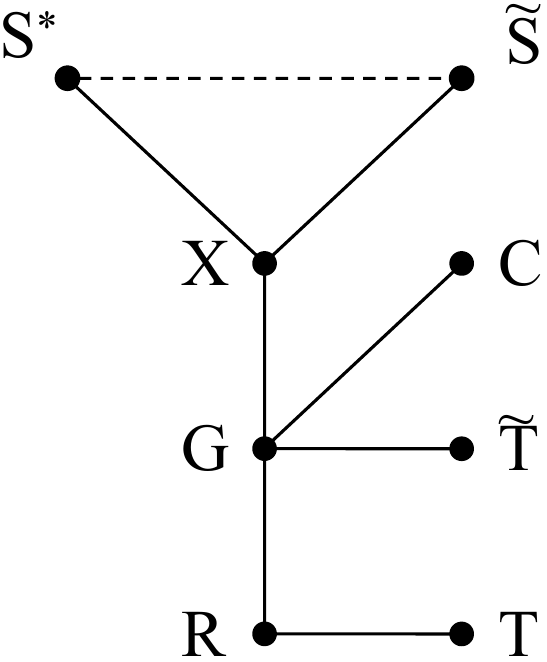} \\
                \label{IC-b}
		\end{minipage}
	}\hspace{10mm}
 	\subfigure[]{
		\begin{minipage}{0.15\textwidth} 
			\includegraphics[width=\textwidth]{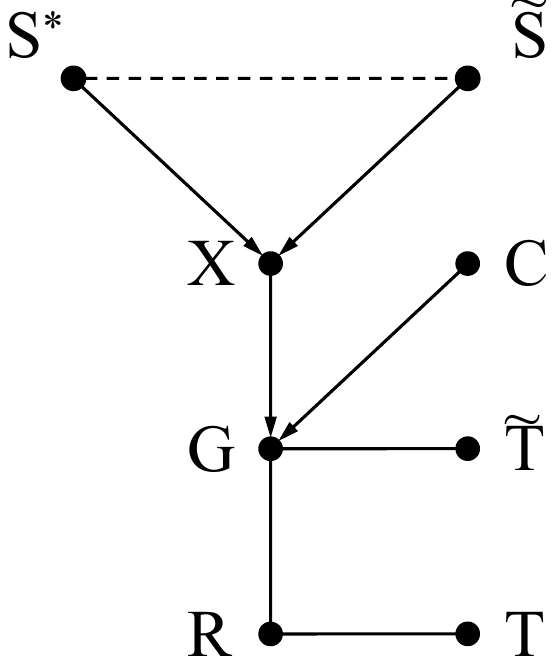} \\
                \label{IC-c}
		\end{minipage}
	}\hspace{10mm}
  	\subfigure[]{
		\begin{minipage}{0.15\textwidth} 
			\includegraphics[width=\textwidth]{IC-d.pdf} \\
                \label{IC-d}
		\end{minipage}
	}

	\caption{Visualization of the reasoning process using IC algorithm.}
	\label{fig:IC}
\end{figure*}

As for step 1, we first represent all variables as nodes in Figure \ref{IC-a}. For each node, we traverse all other nodes to determine whether to establish a connection. For general graph learning scenarios, $R$ represents the model's representation of $G$. Since any other variable can only affect $R$ through $G$, $R$ is conditionally independent of all other variables given $G$. Therefore, $R$ is only connected to $G$. $G$ represents the graph data, which is composed of $X$ and $C$. There is no other variable that can block the relationship between $G$ and $X$, or $G$ and $C$. Therefore, $G$ is connected to both $X$ and $C$. As $X$ blocks the path from $S^*$ and $\widetilde{S}$ to $G$, the corresponding connections do not exist. We define $C$ as a confounder in the graph data, which has no causal relationship with other variables. Therefore, an empty set can make $C$ independent of other variables, except for $G$.

As $\widetilde{S}$ and $S^{*}$ are the factors that decide the form of $X$, each of them is correlated with $X$. Their connections cannot be blocked by any set of variables, and thus we connect all these nodes with $X$. The relationship between $S^{*}$ and $\widetilde{S}$ can not be figured out, therefore we adopt a dashed line to link them. The result is demonstrated in Figure \ref{IC-b}.

We then move on to step 2. Similar to step 1, we process with the traversal analysis starting from $R$. For $R$, since $G \in S_{RG}$ and $G \in S_{RX}$, no edges related to $R$ can be directed. Then, as $S_{XC} = \emptyset$, $X \notin S_{XC}$, we direct edge $X$ to $G$ and $C$ to $G$. As $S^{*}$ and $\widetilde{S}$ is the cause of $X$ by definition, therefore we direct edge $S^{*}$ to $X$ and $\widetilde{S}$ to $X$.


In Step 3, we adopted the rule 1 from \cite{DBLP:conf/uai/VermaP92} for systematizing this step, which states that if $a \to b$ and $a$ and $c$ are not adjacent, then $b \to c$ should be set as the orientation for $b - c$. Based on this, we established the following orientations: $G \to \widetilde{T}, G \to R$. Then, according to the same rule, $R \to T$ should also be set. As for the remaining edge $(\widetilde{S},S^{*})$, since we cannot determine its direction, we represent this edge with a bidirectional dashed line, indicating its directionality is uncertain. The final result is illustrated in Figure \ref{IC-d}, which is identical to the SCM in Figure \ref{fig:SCM}.
\end{proof}


\subsection{Causal Model Based Analysis } 

Based on SCM in Figure \ref{fig:SCM}, we provide theoretical proof to support the intuition that incorporating diminutive causal structure is beneficial for enhancing a model's performance. To achieve such a goal, We propose the following theorem.

\begin{theorem}
\label{th:c}
For a graph representation learning process with a causal structure represented by the SCM in Figure \ref{fig:SCM}, increasing the mutual information $I(R;\widetilde{S})$ between $R$ and $\widetilde{S}$ can decrease the upper bound of the mutual information $I(R;C)$ between $R$ and $C$. Formally:
\begin{gather}
I(R;C) \leq 1-I(R;\widetilde{S}).
\end{gather} 
\end{theorem}

\textit{Proof of Theorem \ref{th:c}.}

To prove the theorem, we follow \cite{pearl2009causality} and suppose the proposed SCM possesses Markov property. Therefore, according to the SCM in Figure \ref{fig:SCM}, $C$ and $R$ are conditionally independent given $G$, as $G$ blocks any path between $C$ and $R$. Therefore, we could apply the Data Processing Inequality \cite{DBLP:books/wi/01/CT2001} on the path $C \to G \to R$. Formally, we have:
\begin{gather}
I(G;C) \geq I(R;C).
\end{gather} 
According to the Chain Rule for Information \cite{DBLP:books/wi/01/CT2001}, we also have:
\begin{gather}
I(X,C;G) = I(X;G|C) + I(C;G).
\end{gather} 
As $X$ and $C$ are independent, we have:
\begin{gather}
I(X,C;G) = I(X;G) + I(C;G).
\end{gather} 
Therefore:
\begin{gather}
 I(C;G) = I(X,C;G) - I(X;G).
\end{gather} 
As $G$ is the graph data that consist of $X$ and $C$, we have:
\begin{gather}
 I(X,C;G) = 1.
\end{gather} 
Then:
\begin{gather}
 I(C;G) = 1 - I(X;G).
\end{gather} 
Therefore:
\begin{gather}
\label{eq:1-I}
1 - I(X;G) \geq I(R;C).
\end{gather} 
According to the SCM in Figure \ref{fig:SCM}, $\widetilde{S}$ and $R$ are conditionally independent given $G$, we have:
\begin{gather}
\label{eq:ptc_sgsr}
I(\widetilde{S};G) \geq I(\widetilde{S};R).
\end{gather} 
Also:
\begin{gather}
I(\widetilde{S},X;G) = I(\widetilde{S};G) + I(X;G|\widetilde{S})
\end{gather} 
and
\begin{gather}
I(\widetilde{S},X;G) = I(X;G) + I(\widetilde{S};G|X)
\end{gather} 
holds. $\widetilde{S}$ and $G$ is independent given $X$, therefore:
\begin{gather}
I(\widetilde{S};G|X) = 0.
\end{gather} 
Thus:
\begin{gather}
I(\widetilde{S},X;G) = I(\widetilde{S};G) + I(X;G|\widetilde{S}) = I(X;G).
\end{gather} 
As $I(X;G|\widetilde{S}) \geq 0$:
\begin{gather}
\label{eq:ptc_sgxg}
I(\widetilde{S};G) \leq I(X;G).
\end{gather} 
Based on Inequality \ref{eq:ptc_sgsr} and \ref{eq:ptc_sgxg}, we have:
\begin{gather}
\label{eq:ptc_sfinall}
I(\widetilde{S};R) \leq I(X;G).
\end{gather} 
Substituting Inequality \ref{eq:ptc_sfinall} into Inequality \ref{eq:1-I}, we can obtain:
\begin{gather}
I(R;C) \leq 1-I(\widetilde{S};R).
\end{gather} 
Theorem \ref{th:c} is proved.

 Theorem \ref{th:c} establishes a direct relationship between a model's alignment with diminutive causal structure and the reduction of confounding influence. As the model gains an increased alignment, the confounding influence diminishes. Next, we prove the validity of the training objective proposed in Equation \ref{eq:Lo}.

\begin{theorem}
\label{th:ori}
For a certain graph learning process represented by the SCM in Figure \ref{fig:SCM}, $\widetilde{T} \sim p(t), T \sim q(t), t\in\mathcal{T}$, if the high-level causal models $r(\cdot)$ is effective enough such that $I(\widetilde{T}, \widetilde{S}) = I(G, \widetilde{S})$, then $I(R, \widetilde{S})$ is maximized if $p(t|\widetilde{s}) = q(t|\widetilde{s})$.
\end{theorem}

\textit{Proof of Theorem \ref{th:ori}.}

We begin with calculating the boundaries of $I(R,\widetilde{S})$. According to the SCM in Figure \ref{fig:SCM}, $\widetilde{S}$ and $G$ are conditionally independent given $X$, as $X$ block any path between $\widetilde{S}$ and $G$. Therefore, we could apply the Data Processing Inequality \cite{DBLP:books/wi/01/CT2001} on the path $\widetilde{S} \to X \to G$. Formally, we have:
\begin{gather}
I(\widetilde{S}; X) \geq I(\widetilde{S};G).
\end{gather} 
Likewise, we have:
\begin{gather}
I(\widetilde{S}; G) \geq I(\widetilde{S};R),
\end{gather} 
and:
\begin{gather}
I(\widetilde{S}; R) \geq I(\widetilde{S};T).
\end{gather} 
According to the assumptions:
\begin{gather}
I(\widetilde{T}; \widetilde{S}) = I(G;\widetilde{S}),
\end{gather} 
we have:
\begin{gather}
I(\widetilde{S} ; \widetilde{T}) \geq I(\widetilde{S};R) \geq I(\widetilde{S};T).
\end{gather} 
So far, we can acquire the upper and lower bounds of $I(\widetilde{S};R)$. Moreover, since the data set and $\mathcal{M}$ are determined in advance in the learning task, the only thing that can be changed for this system is the parameters of the neural network model. Therefore, the upper bound $I(\widetilde{S} ; \widetilde{T})$ holds a fixed value. If we can make the lower bound $I(\widetilde{S};T)$ equal to the upper bound $I(\widetilde{S} ; \widetilde{T})$, then we have $I(\widetilde{S};R)$ reached the maximum. Next, we will proof if $p(t|\widetilde{s}) = q(t|\widetilde{s})$, then $I(\widetilde{S} ; \widetilde{T}) = I(\widetilde{S} ; T)$.

When $p(t|\widetilde{s}) = q(t|\widetilde{s})$, then: 
\begin{align}
H(T|\widetilde{S}) &= - \sum_{\widetilde{s} \in \mathcal{\widetilde{S}}} p_{\widetilde{S}}(\widetilde{s})\sum_{ t \in \mathcal{T}} p(t| \widetilde{s}) log\Big( p(t| \widetilde{s})\Big) \nonumber\\
&=  - \sum_{\widetilde{s} \in \mathcal{\widetilde{S}}} p_{\widetilde{S}}(\widetilde{s})\sum_{ t \in \mathcal{T}} q(t| \widetilde{s}) log\Big( q(t| \widetilde{s})\Big) \nonumber\\
&= H(\widetilde{T}|\widetilde{S}),
\label{TSeq} 
\end{align}
where $\widetilde{S} \sim p_{\widetilde{S}}(\widetilde{s}), \widetilde{s} \in \mathcal{\widetilde{S}}$.
We also have:
\begin{align}
H(T) &=  - \sum_{t \in \mathcal{T}} p(t) log\Big( p(t) \Big) \nonumber\\
&=  - \sum_{t \in \mathcal{T}} \sum_{\widetilde{s} \in \mathcal{\widetilde{S}}}p_{\widetilde{S}}(\widetilde{s})p(t|\widetilde{s}) log\Big( \sum_{\widetilde{s} \in \mathcal{\widetilde{S}}}p_{\widetilde{S}}(\widetilde{s})p(t|\widetilde{s}) \Big) \nonumber\\
&=  - \sum_{t \in \mathcal{T}} \sum_{\widetilde{s} \in \mathcal{\widetilde{S}}}p_{\widetilde{S}}(\widetilde{s})q(t|\widetilde{s}) log\Big( \sum_{\widetilde{s} \in \mathcal{\widetilde{S}}}p_{\widetilde{S}}(\widetilde{s})q(t|\widetilde{s}) \Big) \nonumber\\
&=    - \sum_{t \in \mathcal{T}} q(t) log\Big( q(t) \Big) \nonumber\\
&= H(\widetilde{T})
\label{Teq} 
\end{align}
According to the definition of mutual information \cite{DBLP:books/wi/01/CT2001}, we have:
\begin{gather}
I(\widetilde{S} ; \widetilde{T}) = H(\widetilde{T}) - H(\widetilde{T} | \widetilde{S}), \nonumber\\
I(T ; \widetilde{S}) = H(T) - H(T | \widetilde{S}).
\label{IHH}
\end{gather} 
Based on equation \ref{TSeq}, \ref{Teq} and \ref{IHH}, we have:
\begin{gather}
I(\widetilde{S} ; \widetilde{T}) = I(\widetilde{S} ; T).
\end{gather} 
Based on the discussions above, $I(\widetilde{S};R)$ reached the maximum. Theorem \ref{th:ori} is proved.

According to Theorem \ref{th:ori}, if Equation \ref{eq:Lo} holds, $I(R, \widetilde{S})$ will be maximized, which indicates that the model has learned the maximum amount of knowledge about $\widetilde{S}$. With Theorem \ref{th:ori}, we present a corollary to substantiate the validity of $\mathcal{L}_{c}$.

\begin{corollary}
\label{corollary:Lc2}
For a specific graph dataset $G$, if the training samples in $G$ sufficiently cover all embodiments corresponding to $\widetilde{S}$, and the $\mathcal{L}_{c}$ which is defined by Equation \ref{eq:Lc2} reaches zero, then the maximized value of $I(R, \widetilde{S})$ stated in Theorem \ref{th:ori} can be achieved.
\end{corollary}

\textit{Proof of Corollary \ref{corollary:Lc2}.} 

As in Equation \ref{eq:Lo}, for each graph $G_{i}$, $\widetilde{S}$ holds a fixed value $\widetilde{s}_{i}$. Therefore, we have:
\begin{gather}
p_{i}(t) = p_{i}(t|\widetilde{s}_{i}) = p_{i}(t|\widetilde{s}),
\end{gather} 
and 
\begin{gather}
\hat{q}_{i}(t) = \hat{q}_{i}(t|\widetilde{s}_{i}) = \hat{q}_{i}(t|\widetilde{s}).
\end{gather} 
Therefore, we have:
\begin{gather}
\mathcal{L}_{c} = \sum_{i}^{N} KL\Big(p_{i}(t|\widetilde{s}) || \hat{q}_{i}(t|\widetilde{s})\Big) = \sum_{i}^{N} KL\Big(p_{i}(t) || \hat{q}_{i}(t)\Big)
\end{gather} 
As $r^{P}(r^{E}(G_{i}))$ can predict $p_{i}(t)$. $f^{P}(f^{E}(G_{i}))$ output a estimation $\hat{q}_{i}(t)$ of ${q}_{i}(t)$. Therefore:

\begin{gather}
\mathcal{L}_{c} = \sum_{i}^{N} KL\Big(p_{i}(t) || \hat{q}_{i}(t)\Big) \nonumber\\ = \sum_{i=1}^{N} \mathcal{D}_{KL}\Big(r^{P}(r^{E}(G_i)), f^{P}(f^{E}(G_i))\Big)
\end{gather} 

Then, if:

\begin{gather}
\mathcal{L}_{c} =  \sum_{i=1}^{N} \mathcal{D}_{KL}\Big(r^{P}(r^{E}(G_i)), f^{P}(f^{E}(G_i))\Big) = 0,
\end{gather} 

we have:

\begin{gather}
\sum_{i}^{N} KL\Big(p_{i}(t) || \hat{q}_{i}(t)\Big) = \sum_{i}^{N} KL\Big(p_{i}(t|\widetilde{s}) || \hat{q}_{i}(t|\widetilde{s})\Big) = 0.
\end{gather} 

Further more, the training samples in $G$ sufficiently cover all embodiments corresponding to $\widetilde{S}$, we have:

\begin{gather}
KL\Big(p(t|\widetilde{s}) || q(t|\widetilde{s})\Big) = 0.
\end{gather} 

Then $p(t|\widetilde{s})$ equals $q(t|\widetilde{s})$, according to Theorem \ref{th:ori}, $I(R;\widetilde{S})$ reaches maximization. Corollary \ref{corollary:Lc2} is proved.

\section{Experiments}

\begin{table}[h] \scriptsize
	\centering
 	\caption{Summary of datasets used for experiment with diminutive causal structure from the topological domain.}
    \vskip 0.1in
	\begin{tabular}{lccccccc}
		\toprule
		\multirow{2}*{Name}      &  \multirow{2}*{Graphs} & Average & \multirow{2}*{Classes} & \multirow{2}*{Task Type} & \multirow{2}*{Metric}       \\
  		     &  & Nodes &  & &       \\
		\midrule
		Spurious-Motif &18,000 &46.6 &3 &Classification &ACC \\
		Motif-Variant  &18,000 &48.9 &3 &Classification &ACC \\
		\bottomrule
	\end{tabular}
	\label{tab:datasets1}
\end{table}

\begin{table*}[ht]\scriptsize
	\setlength{\tabcolsep}{8pt}
 	\caption{Performance of graph classification accuracy in Spurious-Motif. Spurious-Motif (ID) denotes the ID version dataset. Bias represents the degree of distribution shift between the training set and the test set. Some of the results are cited from \cite{DBLP:conf/iclr/WuWZ0C22}. The best records are highlighted in \textbf{bold}, and \underline{underline} denotes the second-best result.}
   \vskip 0.1in
	\centering
		\begin{tabular}{l|ccccccc|c}
			\hline\rule{0pt}{8pt}

			\multirow{2}*{Method}   & \multicolumn{7}{c}{Spurious-Motif } \vline&  Spurious-Motif   \\
			\cline{2-8}\rule{0pt}{8pt} 
			& Balanced & bias = 0.4 & bias = 0.5 & bias = 0.6 & bias = 0.7 & bias = 0.8 & bias = 0.9 & (ID) \\ 
			\hline\rule{0pt}{8pt}
			\text{Local Extremum GNN \cite{DBLP:conf/aaai/RanjanST20}}  & 42.99$\pm$1.93 & 40.36$\pm$0.95 & 39.69$\pm$1.73 & 39.13$\pm$1.38 & 38.93$\pm$1.74 & 36.76$\pm$1.59 & 33.61$\pm$1.02 & {92.91$\pm$1.98}\\ \rule{0pt}{8pt}
			\text{GAT \cite{velivckovic2017graph}} & 43.07$\pm$2.55 & 41.82$\pm$1.96 & 39.42$\pm$1.50 & 37.98$\pm$1.63 & 37.41$\pm$0.86 & 35.68$\pm$0.54 & 33.46$\pm$0.43 &  {91.31$\pm$2.28}\\  \rule{0pt}{8pt}
			\text{Top-k Pool \cite{gao2019graph}}  &  43.43$\pm$8.79 & 42.85$\pm$7.34 & 41.21$\pm$7.05 & 40.90$\pm$6.95 &  40.27$\pm$7.12 & 36.84$\pm$3.48 & 33.60$\pm$0.91 & {90.83$\pm$3.04}\\  \rule{0pt}{8pt}
            \text{GSN \cite{DBLP:journals/pami/BouritsasFZB23}} &   43.18$\pm$5.65 & 38.53$\pm$3.49 &  34.67$\pm$1.21 & 34.56$\pm$0.38 &  34.03$\pm$1.69 & 33.48$\pm$1.57 &  32.60$\pm$1.75 & {92.81$\pm$1.10}\\  \rule{0pt}{8pt}
			\text{Group DRO \cite{sagawa2019distributionally}} &  41.51$\pm$1.11 & 40.48$\pm$1.22 & 39.38$\pm$0.93 & 39.68$\pm$1.56 & 39.32$\pm$2.23 & 35.15$\pm$1.19 & 33.90$\pm$0.52 & {92.18$\pm$1.12}\\  \rule{0pt}{8pt}
			\text{IRM \cite{arjovsky2019invariant}}   & 42.26$\pm$2.69 & 41.94$\pm$1.90 &  41.30$\pm$1.28 & 40.96$\pm$1.65 & 40.16$\pm$1.70 & 38.53$\pm$1.86 &  35.12$\pm$2.71 & {91.12$\pm$1.56}\\  \rule{0pt}{8pt}
			\text{V-REx \cite{krueger2021out}}  & 42.83$\pm$1.59 & 40.36$\pm$1.99 &  39.43$\pm$2.69 & 39.33$\pm$1.70 & 39.08$\pm$1.56 & 36.85$\pm$1.86  &34.81$\pm$2.04  & {91.08$\pm$1.85}\\  \rule{0pt}{8pt}
            \text{DIR \cite{DBLP:conf/iclr/WuWZ0C22}}    & 43.53$\pm$3.38 & 42.99$\pm$3.35 & 41.45$\pm$2.12 & 41.37$\pm$2.08 & 41.03$\pm$1.53 & 40.26$\pm$1.36 & 39.20$\pm$1.94 & \underline{93.02$\pm$1.89}\\ \rule{0pt}{8pt}
            \text{DISC}	\cite{fan2022debiasing} & 44.93$\pm$2.18 & 44.49$\pm$3.35 & 43.53$\pm$1.93 & 42.78$\pm$2.01 & \underline{42.40$\pm$1.53} & \underline{42.06$\pm$1.68} & \underline{41.35$\pm$1.23} & 92.79$\pm$1.73\\ \rule{0pt}{8pt}
            \text{RCGRL \cite{DBLP:conf/aaai/GaoLQSXZ023}}    & 44.03$\pm$1.98 & 42.96$\pm$1.63 & 42.50$\pm$1.53 & 42.41$\pm$1.68 & 42.02$\pm$1.58 & 41.62$\pm$1.35 & 41.03$\pm$1.12 & 92.65$\pm$1.67\\  
      	    \hline\rule{0pt}{8pt}
            DCS-Only-D  & 42.38$\pm$1.93  &  40.53$\pm$1.36 &  39.93$\pm$1.89 & 39.43$\pm$1.31  & 39.01$\pm$1.22 &  35.83$\pm$1.06  &   33.32$\pm$0.63 & 91.93$\pm$0.95 \\  \rule{0pt}{8pt}
            DCS-Only-L  & 42.30$\pm$1.63 & 39.93$\pm$1.57  &  38.89$\pm$1.39 & 38.58$\pm$1.74  &  37.41$\pm$1.69 &  35.96$\pm$0.96 &  33.26$\pm$1.20  & 92.02$\pm$1.68 \\  
            \rowcolor{orange!20}\rule{0pt}{8pt}            
            Average  & 42.35 & 40.23  & 39.41  &  39.01  & 38.21 & 35.90 & 33.29 & 91.98  \\  
			\hline\rule{0pt}{8pt}
            DCSGL-T  & 44.38$\pm$2.20 & 43.68$\pm$1.65 & 42.94$\pm$1.98 & 42.36$\pm$1.67 & 41.35$\pm$1.58  & 40.93$\pm$1.68  &  39.06$\pm$1.97  & 92.03$\pm$1.32  \\ \rule{0pt}{8pt}
            DCSGL-A & \underline{45.08$\pm$1.36} & \underline{44.68$\pm$1.54}  &   \underline{43.65$\pm$1.84} & \underline{42.80$\pm$1.79} &  41.98$\pm$1.54 & 41.30$\pm$1.68 &  40.13$\pm$1.89   & 92.87$\pm$1.28 \\  \rule{0pt}{8pt}
			\textbf{DCSGL}   & \bf{47.32$\pm$1.89} &   \bf{46.05$\pm$1.85} & \bf{45.53$\pm$1.64} & \bf{44.01$\pm$1.78} & \bf{43.91$\pm$1.89} & \bf{43.51$\pm$1.99} & \bf{41.68$\pm$1.67}  & \bf{95.38$\pm$1.28}\\ 
			\hline
		\end{tabular}

	\label{tab:sm}
\end{table*}

\begin{table*}[ht]\scriptsize
    \vskip 0.1in
	\setlength{\tabcolsep}{8pt}
	\centering
 	\caption{Performance of graph classification accuracy in Motif-Variant. Motif-Variant (ID) denotes the ID version dataset. Bias represents the degree of distribution change between the training set and the test set. The best records are highlighted in \textbf{bold}, and \underline{underline} denotes the second-best result.}
    \vskip 0.1in
		\begin{tabular}{l|ccccccc|c}
			\hline\rule{0pt}{8pt}
			\multirow{2}*{Method}   & \multicolumn{7}{c}{Motif-Variant }   \vline & Motif-Variant\\
			\cline{2-8}\rule{0pt}{8pt} 
			& Balanced & bias = 0.4 & bias = 0.5 & bias = 0.6 & bias = 0.7 & bias = 0.8 & bias = 0.9 & (ID)\\ 	
			\hline\rule{0pt}{8pt}
			\text{Local Extremum GNN \cite{DBLP:conf/aaai/RanjanST20}}   & 48.18$\pm$3.46 & 46.98$\pm$2.98 & 46.38$\pm$2.52 & 45.98$\pm$2.36 & 45.78$\pm$2.81 & 42.84$\pm$2.20 & 42.31$\pm$2.13 & {94.78$\pm$1.07}\\\rule{0pt}{8pt} 
			\text{GAT \cite{velivckovic2017graph}}  & 49.13$\pm$2.96 & 48.56$\pm$2.62 & 47.78$\pm$2.12 & 46.35$\pm$1.85 & 45.63$\pm$1.93 & 43.06$\pm$1.50 & 41.48$\pm$0.85 & {93.32$\pm$2.60}\\\rule{0pt}{8pt} 
			\text{Top-k Pool \cite{gao2019graph}} & 48.56$\pm$7.10 & 47.35$\pm$7.28 & 46.08$\pm$7.93 & 45.55$\pm$1.68 & 44.37$\pm$8.07 &  43.44$\pm$1.65 & 42.10$\pm$6.13 & {93.17$\pm$3.21} \\\rule{0pt}{8pt} 
            \text{GSN \cite{DBLP:journals/pami/BouritsasFZB23}} & 47.05$\pm$6.03 & 45.87$\pm$4.33 & 44.08$\pm$2.32 & 44.15$\pm$3.65 & 43.15$\pm$1.68 & 41.73$\pm$1.55 & 40.19$\pm$1.93 & {93.90$\pm$2.11}\\\rule{0pt}{8pt} 
			\text{Group DRO \cite{sagawa2019distributionally}}  & 44.03$\pm$1.37 & 43.35$\pm$1.46 & 42.06$\pm$1.21 & 41.48$\pm$1.74 & 41.15$\pm$1.56 & 40.11$\pm$1.82 & 39.90$\pm$1.02 & {92.54$\pm$1.31}\\\rule{0pt}{8pt} 
			\text{IRM \cite{arjovsky2019invariant}}   & 48.13$\pm$2.86 & 46.40$\pm$2.43 & 44.30$\pm$1.49 & 43.61$\pm$1.81 & 42.01$\pm$2.03 & 41.55$\pm$2.50 & 40.18$\pm$2.33 & {93.17$\pm$1.21}\\\rule{0pt}{8pt} 
			\text{V-REx \cite{krueger2021out}}  & 49.76$\pm$1.68 & 48.80$\pm$1.69 & 46.83$\pm$2.36 & 44.34$\pm$1.34 & 43.12$\pm$1.50 & 42.67$\pm$1.50 & 42.37$\pm$1.99  & {93.13$\pm$1.25}\\\rule{0pt}{8pt} 
			\text{DIR \cite{DBLP:conf/iclr/WuWZ0C22}}   & 49.66$\pm$2.85 & 48.91$\pm$2.31 & 47.76$\pm$2.73 & 45.31$\pm$1.60 & 44.80$\pm$1.32 & 44.12$\pm$1.96 & 42.90$\pm$1.68 & {94.46$\pm$1.57} \\\rule{0pt}{8pt} 
      		\text{DISC \cite{fan2022debiasing}} & \underline{51.38$\pm$1.97} & \underline{49.65$\pm$1.98} & \underline{49.08$\pm$1.65} & \underline{48.36$\pm$1.76} & 46.76$\pm$1.89 & 45.56$\pm$1.55 & 43.96$\pm$1.68  & \bf{95.98$\pm$1.70} \\\rule{0pt}{8pt} 
   			\text{RCGRL \cite{DBLP:conf/aaai/GaoLQSXZ023}} & 50.01$\pm$2.79   & 49.30$\pm$2.20 & 48.93$\pm$1.65 & 48.12$\pm$1.80 & \underline{47.13$\pm$1.95} & \underline{46.93$\pm$1.73} & 44.33$\pm$1.55  & \bf{95.98$\pm$1.70} \\
			\hline\rule{0pt}{8pt}
            DCS-Only-D & 48.03$\pm$1.84  & 47.57$\pm$1.57 & 45.41$\pm$1.77 & 44.93$\pm$1.94 & 44.52$\pm$2.13 & 43.66$\pm$1.48 & 42.78$\pm$1.95 & 92.79$\pm$1.88 \\\rule{0pt}{8pt} 
            DCS-Only-L & 47.98$\pm$1.93 & 47.20$\pm$1.99 & 44.70$\pm$1.81 & 43.79$\pm$1.49 & 43.65$\pm$1.67 & 42.91$\pm$1.45 & 41.98$\pm$1.64  & 93.69$\pm$1.93 \\
            \rowcolor{orange!20}\rule{0pt}{8pt}            
            Average & 48.01 & 47.39 & 45.06 & 44.36 & 44.09 & 43.29 & 42.38 & 93.24 \\
            \hline\rule{0pt}{8pt}
            DCSGL-T  & 49.98$\pm$1.71 &  48.84$\pm$1.90 &  \underline{48.95$\pm$1.91} & 47.30$\pm$1.58 & 46.23$\pm$1.79 & 45.87$\pm$1.65 & 45.13$\pm$1.75  & 94.05$\pm$0.87 \\\rule{0pt}{8pt} 
    		DCSGL-A  & 49.06$\pm$1.74 & 48.33$\pm$1.70 & 47.45$\pm$1.23 & 47.35$\pm$1.03 & 46.06$\pm$1.78 & 46.15$\pm$1.69 & \underline{45.89$\pm$1.75}  & 93.88$\pm$0.98 \\\rule{0pt}{8pt} 
			\textbf{DCSGL}  & \bf{52.16$\pm$1.88} & \bf{50.09$\pm$1.63} & \bf{49.00$\pm$1.30} & \bf{48.55$\pm$1.85} & \bf{48.33$\pm$1.93} & \bf{48.13$\pm$1.78} & \bf{47.06$\pm$1.87} & \underline{95.53$\pm$0.93} \\ 
			\hline
		\end{tabular}

	\label{tab:smv}
\end{table*}

\begin{table}[ht]\scriptsize
    \vskip 0.1in
	\setlength{\tabcolsep}{5pt}
	\centering
 	\caption{Performance of classification accuracy in node classification tasks. Some of the baselines are removed as they are specifically designed to address graph classification tasks. Mixed-N denotes the graph data mixed with motif from Spurious-Motif and Motif-Variant. The best records are highlighted in \textbf{bold}, and \underline{underline} denotes the second-best result. }
    \vskip 0.1in
		\begin{tabular}{l|ccc}
			\hline\rule{0pt}{8pt}
			Method  & Spurious-Motif-N & Motif-Variant-N & Mixed-N \\
			\hline\rule{0pt}{8pt}
			\text{Local Extremum GNN \cite{DBLP:conf/aaai/RanjanST20}}   & 76.24$\pm$1.70 & 75.13$\pm$1.98 & 74.58$\pm$1.63 \\\rule{0pt}{8pt} 
			\text{GAT \cite{velivckovic2017graph}}  & 77.31$\pm$1.63 & 73.62$\pm$1.53 & 73.78$\pm$1.32 \\\rule{0pt}{8pt} 
			\text{Group DRO \cite{sagawa2019distributionally}}  & 75.01$\pm$1.89 & 73.10$\pm$1.55 & 74.03$\pm$1.39 \\\rule{0pt}{8pt} 
			\text{IRM \cite{arjovsky2019invariant}}   & 74.86$\pm$1.91 & 75.40$\pm$1.89 & 72.30$\pm$1.82\\\rule{0pt}{8pt} 
			\text{V-REx \cite{krueger2021out}}  & 76.67$\pm$1.46 & 75.58$\pm$1.99 & 74.63$\pm$1.47 \\\rule{0pt}{8pt} 
			\text{DIR \cite{DBLP:conf/iclr/WuWZ0C22}}   & 76.50$\pm$1.68 & 75.66$\pm$2.01 & 76.30$\pm$1.53 \\\rule{0pt}{8pt} 
   			\text{RCGRL \cite{DBLP:conf/aaai/GaoLQSXZ023}} &75.95$\pm$1.50  & 76.85$\pm$1.43 & 74.90$\pm$1.65 \\
			\hline\rule{0pt}{8pt}
            DCS-Only-D & 76.12$\pm$1.69  & 75.65$\pm$1.71 & 75.68$\pm$1.70 \\\rule{0pt}{8pt} 
            DCS-Only-L & 76.95$\pm$1.87 & 75.38$\pm$1.65 & 74.80$\pm$1.79 \\
            \rowcolor{orange!20}\rule{0pt}{8pt}            
            Average & 76.54 & 75.52 & 75.24 \\
            \hline\rule{0pt}{8pt}
            DCSGL-T  & \underline{78.37$\pm$1.86} &  \underline{78.91$\pm$1.49} &  77.31$\pm$1.09 \\\rule{0pt}{8pt} 
    		DCSGL-A  & 78.01$\pm$1.50 & 77.40$\pm$1.66 & \underline{77.38$\pm$1.35}  \\\rule{0pt}{8pt} 
			\textbf{DCSGL}  & \bf{80.16$\pm$1.73} & \bf{80.01$\pm$1.70} & \bf{78.08$\pm$1.54} \\ 
			\hline
		\end{tabular}

	\label{tab:sn}
\end{table}

In this section, we conduct a comprehensive analysis of the performance of our approach through a series of experiments. We extract distinct diminutive causal structures in two prevalent domains of graph representation learning: graph topology analysis and text graph analysis. Validation of our method is carried out across multiple datasets within each domain. Additionally, we perform a series of further analyses to delve deeper into the intrinsic mechanisms of DCSGL.

\subsection{Experiments with Diminutive Causal Structure from Topological Domain}
For the study of graph representation learning, the topological structure of the graph is particularly crucial, encompassing a wealth of essential information \cite{DBLP:conf/nips/YangWY22, DBLP:conf/icml/LiWLCX22}. In this experiment, we undertake diminutive causal structure discovery based on knowledge pertaining to graph topological structures and construct $\mathcal{M}(\cdot)$ and $\mathcal{M}^{\gamma}(\cdot)$. Subsequently, a series of experiments were conducted on artificially synthesized datasets for validation.

\begin{figure}[h]
    \centering
    \includegraphics[width=0.35\textwidth]{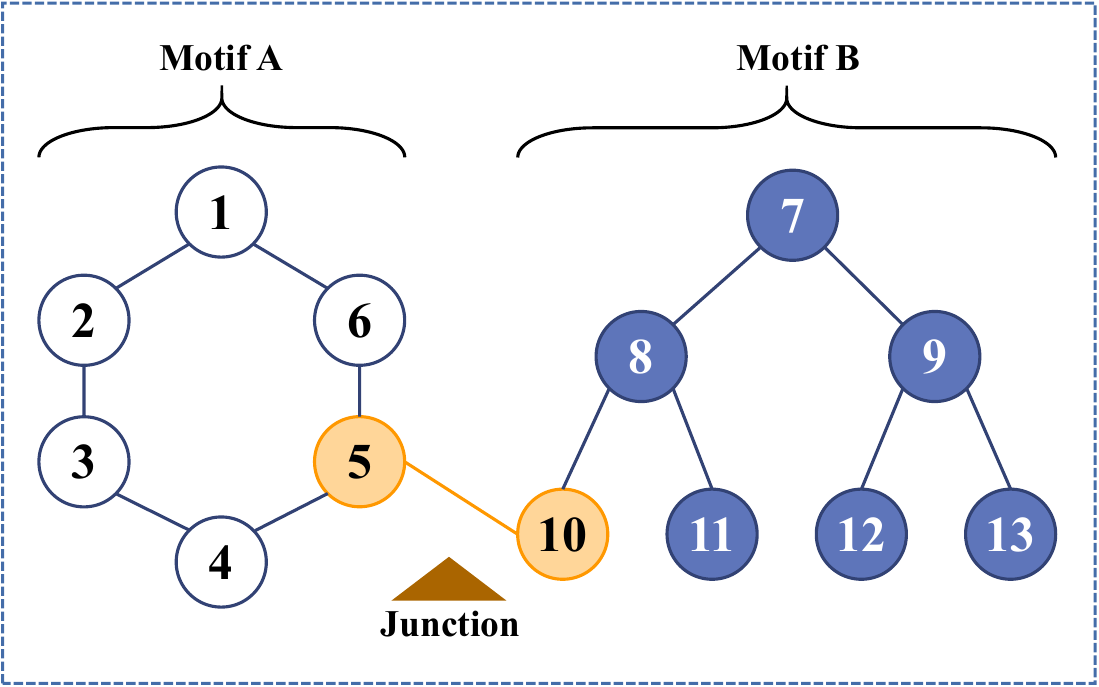}
    \caption{Example of motif junction. In the graph, the white nodes and blue nodes respectively represent two distinct motifs, while the junction of these two motifs is indicated in orange.}
    \label{fig:jun}
\end{figure}

\subsubsection{Dataset}
The datasets employed in this experiment include the following: 1) Spurious-Motif, a synthetic Out-Of-Distribution (OOD) dataset \cite{DBLP:conf/iclr/WuWZ0C22} generated using the methodology from \cite{DBLP:conf/nips/YingBYZL19}; 2) Motif-Variant, another synthetic OOD dataset created by ourselves, employing the data generation method outlined in \cite{DBLP:conf/nips/YingBYZL19}, but featuring distinct topological structures compared to the previous dataset; 3) The In-Distribution (ID) counterparts of these datasets. 4) Spurious-Motif-N and Motif-Variant-N, The node classification versions for Spurious-Motif and Motif-Variant. Nodes within different motifs are assigned distinct labels to evaluate the effectiveness of node classification. Detailed information about the datasets is provided in Table \ref{tab:datasets1}.

\subsubsection{Model $\mathcal{M}(\cdot)$ Construction}
\label{sec:mds1}

\begin{table}[h] \scriptsize
	\centering
 	\caption{Summary of datasets used for experiment with diminutive causal structure from the topological domain.}
    \vskip 0.1in
	\begin{tabular}{lccccccc}
		\toprule
		\multirow{2}*{Name}      &  \multirow{2}*{Graphs} & Average & \multirow{2}*{Classes} & \multirow{2}*{Task Type} & \multirow{2}*{Metric}       \\
  		     &  & Nodes &  & &       \\
		\midrule
        Graph-SST2 &70,042 &10.2 &2 &Classification &ACC \\
		Graph-SST5 &11,550 &21.1 &5 &Classification &ACC \\
		Graph-Twitter &6,940 &19.8 &3 &Classification &ACC \\
		\bottomrule
	\end{tabular}
	\label{tab:datasets2}
\end{table}

\begin{table*}[ht]\scriptsize
    \setlength{\tabcolsep}{10pt}
	\centering
 	\caption{Performance of graph classification accuracy in Graph-SST2, Graph-SST5, and Graph-Twitter. `(OOD)' denotes the OOD version datasets. The best records are highlighted in \textbf{bold}, and \underline{underline} denotes the second-best result.}
    \vskip 0.1in
		\begin{tabular}{l|cc||cc||cc}
			\hline\rule{0pt}{8pt}
			\multirow{2}*{Method} & \multirow{2}*{Graph-SST2} & Graph-SST2 & \multirow{2}*{Graph-SST5} & Graph-SST5 & \multirow{2}*{Graph-Twitter} & Graph-Twitter \\\rule{0pt}{8pt} 	
			  & & (OOD) &  & (OOD) &  & (OOD) \\ 
			\hline\rule{0pt}{8pt}
			\text{ARMA\cite{DBLP:conf/aaai/0001RFHLRG19}} & 89.19±0.87 & 81.44$\pm$0.59& 48.13$\pm$0.98 & 37.26$\pm$0.89 & 63.19$\pm$1.13 & 62.56$\pm$1.65 \\\rule{0pt}{8pt}
			\text{GAT \cite{velivckovic2017graph}}  & 89.89±0.68 & 81.57$\pm$0.71 & 48.51$\pm$1.86 & 37.58$\pm$1.67 & 63.57$\pm$0.95 & 62.38$\pm$0.98 \\\rule{0pt}{8pt} 
			\text{Top-k Pool \cite{gao2019graph}}  & 89.31±1.26 & 79.78$\pm$1.35 &  \underline{49.72$\pm$1.42} & 36.26$\pm$1.86 & 63.41$\pm$1.95 & 62.95$\pm$1.09 \\\rule{0pt}{8pt}
            \text{GSN \cite{DBLP:journals/pami/BouritsasFZB23}}  & 89.63$\pm$1.60 & 81.93$\pm$1.86 &  48.64$\pm$1.60 & \underline{38.78$\pm$1.84} & 63.18$\pm$1.89 & 63.07$\pm$1.18 \\\rule{0pt}{8pt}
			\text{Group DRO \cite{sagawa2019distributionally}}  & 89.94±1.60 & 81.29$\pm$1.44 &   47.44$\pm$1.12 & 37.78$\pm$1.12 & 62.23$\pm$1.43 & 61.90$\pm$1.03 \\\rule{0pt}{8pt}
			\text{IRM \cite{arjovsky2019invariant}}   & 89.55±1.03 & 81.01$\pm$1.13 & 48.08$\pm$1.30 &  38.68$\pm$1.62 &  64.11$\pm$1.58 &  62.27$\pm$1.55 \\\rule{0pt}{8pt}
			\text{V-REx \cite{krueger2021out}}  & 88.78±0.82 & 81.76$\pm$0.08 & 48.55$\pm$1.10&  37.10$\pm$1.18 &  64.84$\pm$1.46 &  63.42$\pm$1.06 \\\rule{0pt}{8pt}
			\text{DIR \cite{DBLP:conf/iclr/WuWZ0C22}}   & 89.91±0.86 & 81.93$\pm$1.26 & 49.16$\pm$1.31 &  38.67$\pm$1.38 &  65.14$\pm$1.37 &  \underline{63.49$\pm$1.36} \\\rule{0pt}{8pt}
      		\text{DISC \cite{DBLP:conf/aaai/GaoLQSXZ023}}   & 88.95$\pm$0.79 & 81.67$\pm$1.09 & 48.03$\pm$1.43 &  36.61$\pm$1.83 & 65.28$\pm$0.61   & 62.43$\pm$1.26 \\\rule{0pt}{8pt}
   			\text{RCGRL \cite{DBLP:conf/aaai/GaoLQSXZ023}}   & \underline{90.73$\pm$0.40} & \underline{82.31$\pm$1.01} & 48.56$\pm$1.53 &  37.98$\pm$1.63 & \underline{66.58$\pm$0.53}   & 62.93$\pm$1.35 \\
            \hline\rule{0pt}{8pt}
            DCS-Only-D   & 88.93$\pm$0.99 & 80.56$\pm$0.63 & 48.03$\pm$1.04 & 35.86$\pm$1.16 & 63.20$\pm$1.32 &  62.38$\pm$1.75\\\rule{0pt}{8pt}
            DCS-Only-L   & 89.25$\pm$1.15 & 81.03$\pm$0.85 & 48.30$\pm$1.07 & 36.32$\pm$0.77 & 63.30$\pm$1.09 & 61.80$\pm$1.85 \\
            \rowcolor{orange!20}\rule{0pt}{8pt}
            Average   & 89.09 & 80.80 & 48.17 & 36.09 & 63.25 & 62.09 \\
			\hline\rule{0pt}{8pt}
            DCSGL-T   & 89.73$\pm$0.77 & 82.31$\pm$0.97 & 49.20$\pm$0.89 & 38.92$\pm$1.10 & 63.26$\pm$1.22 & 63.00$\pm$1.43 \\\rule{0pt}{8pt}
            DCSGL-A   & 89.68$\pm$0.73 & 82.01$\pm$0.83 & 49.63$\pm$0.89 &  39.00$\pm$1.25 &  65.32$\pm$0.93 &  
               61.98$\pm$0.96  \\\rule{0pt}{8pt}
			\textbf{DCSGL}   & \bf{91.12$\pm$1.53} & \bf{83.13$\pm$1.33} & \bf{50.38$\pm$0.87} &  \bf{42.13$\pm$1.50} &  \bf{66.87$\pm$0.89} &  \bf{63.98$\pm$1.20}  \\
			\hline
		\end{tabular}

	\label{tab:graph}
\end{table*}

\begin{table}[h] \scriptsize
	\centering
 	\caption{Summary of computation cost}.
	\begin{tabular}{l|cc|cc}
		\hline\rule{0pt}{8pt}
		\multirow{2}*{Method}      &  Time & Additional & Memory &        
            Additional      \\\rule{0pt}{8pt}
  		 &  Cost (s)   & Time (s)  & Cost (Gb) & Memory (Gb)      \\
		\hline\rowcolor{gray!20}
            \multicolumn{5}{c}{Results on Spurious-Motif Dataset} \\
            \hline\rule{0pt}{8pt}
            Backbone &2.18 &- &1.98 &-  \\\rule{0pt}{8pt}
            DIR &4.31 &2.13 &3.61 &1.63 \\\rule{0pt}{8pt}
		  DCSGL &4.01 &1.83 &2.56 &0.58  \\
            \hline\rowcolor{gray!20}
            \multicolumn{5}{c}{Results on Graph-SST2 Dataset} \\
            \hline\rule{0pt}{8pt}
            Backbone &17.03 &- &9.13 &-  \\\rule{0pt}{8pt}
            DIR &50.31 &33.28 &13.64 &4.51 \\\rule{0pt}{8pt}
		  DCSGL &23.01 &5.98 &10.16 &1.03 \\
		\hline
	\end{tabular}
	\label{tab:tc}
\end{table}

\begin{table}[ht]\scriptsize
    \vskip 0.1in
	\setlength{\tabcolsep}{5pt}
	\centering
 	\caption{Performance with different backbones. $\Delta$ represents the improvement in accuracy after incorporating our method.}
    \vskip 0.1in
		\begin{tabular}{l|cc|c}
			\hline\rule{0pt}{8pt} 
			\multirow{2}*{Method} & \multicolumn{2}{c}{Motif-Variant } \vline & Motif-Variant \\
            \cline{2-3}\rule{0pt}{8pt} 
            & Balanced & Bias=0.9 & ID \\
			\hline\rule{0pt}{8pt}
			\text{DCSGL-Backbone}   & 48.18$\pm$3.46 & 42.31$\pm$2.13 & 94.78$\pm$1.07 \\\rule{0pt}{8pt} 
			\text{DCSGL }  & 52.16$\pm$1.88 & 47.06$\pm$1.87 & 95.53$\pm$0.93 \\
            \rowcolor{blue!20}\rule{0pt}{8pt}
            \text{$\Delta$ }  & 3.98 & 4.75 & 0.75 \\
            \hline\rule{0pt}{8pt}
            \text{FAGCN \cite{DBLP:conf/aaai/BoWSS21}}   & 46.59$\pm$1.46 & 38.96$\pm$2.03 & 91.56$\pm$1.82 \\\rule{0pt}{8pt} 
			\text{DCSGL-FAGCN}  & 50.16$\pm$1.87 & 41.60$\pm$1.66 & 91.79$\pm$1.85 \\
            \rowcolor{blue!20}\rule{0pt}{8pt}
            \text{$\Delta$ }  & 3.57 & 2.64 & 0.23 \\
            \hline\rule{0pt}{8pt}
            \text{EGC \cite{DBLP:conf/iclr/TailorOLL22}}   & 43.13$\pm$1.97 & 36.53$\pm$1.65 & 92.12$\pm$1.38 \\\rule{0pt}{8pt} 
			\text{DCSGL-EGC}  & 45.69$\pm$1.90 & 39.68$\pm$1.88 & 93.89$\pm$1.65 \\
            \rowcolor{blue!20}\rule{0pt}{8pt}
            \text{$\Delta$ }  & 2.56 & 3.15 & 1.77 \\
            \hline
		\end{tabular}

	\label{tab:multibaseline}
\end{table}

\begin{figure*}[h]
	\centering
	\subfigure[Results on Motif-Varient Bias = 0.9.]{
		\begin{minipage}{0.3\textwidth} 
			\includegraphics[width=\textwidth]{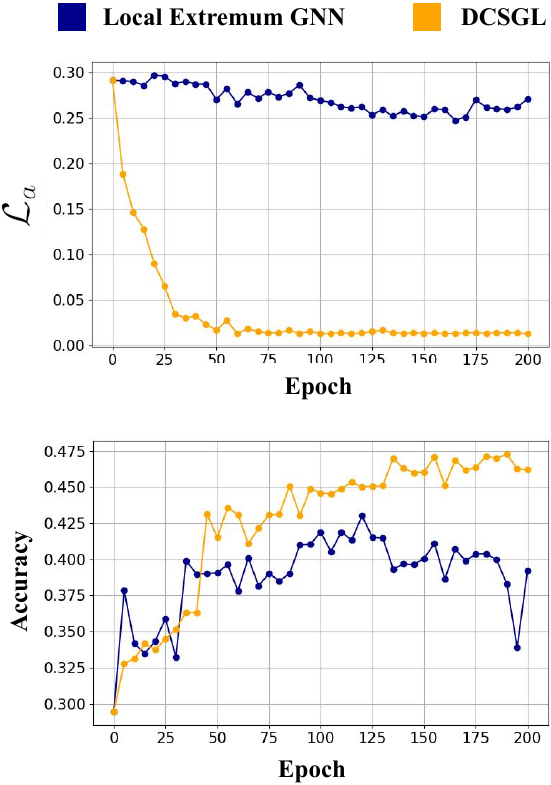} \\
                \label{fig:line_charta}
		\end{minipage}
    }\hspace{2mm}
	\subfigure[Results on Motif-Varient Bias = 0.6.]{
		\begin{minipage}{0.32\textwidth} 
			\includegraphics[width=\textwidth]{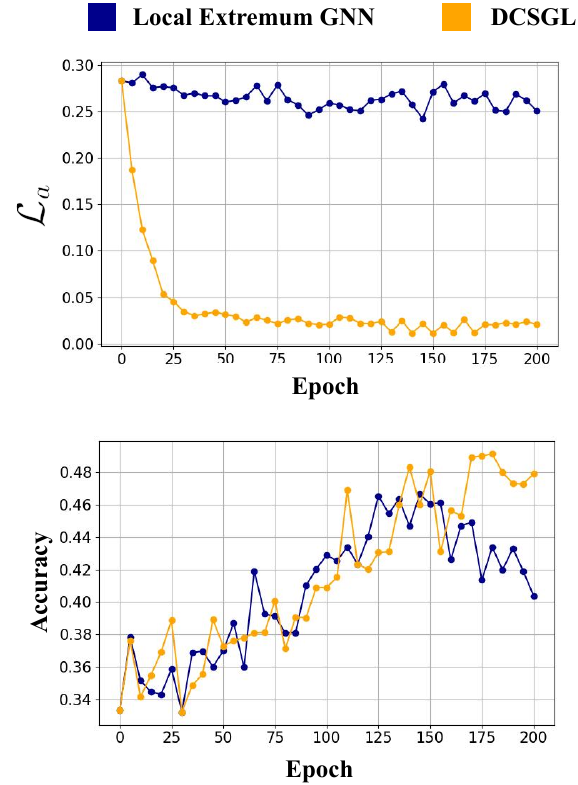} \\
                \label{fig:line_chartb}
		\end{minipage}
	}\hspace{2mm}
 	\subfigure[Results on Motif-Varient Balanced.]{
		\begin{minipage}{0.32\textwidth} 
			\includegraphics[width=\textwidth]{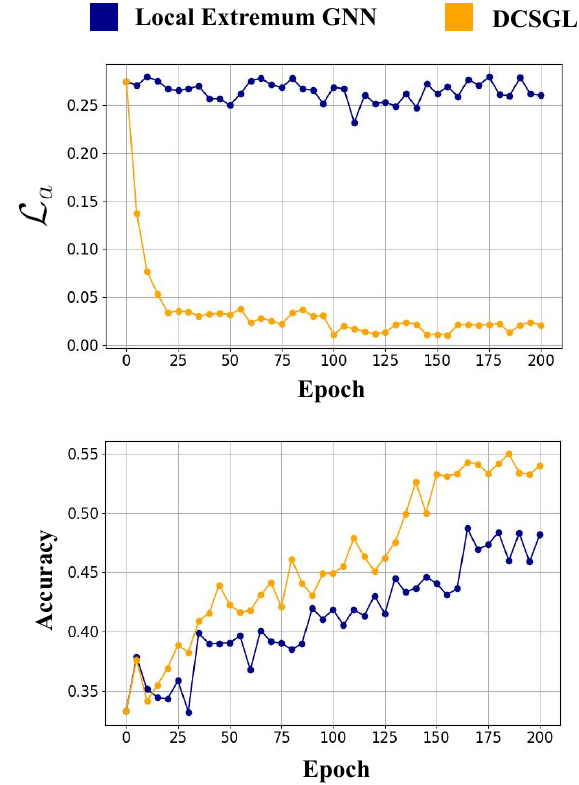} \\
                \label{fig:line_chartc}
		\end{minipage}
	}

	\caption{Accuracy and loss obtained over 200 epochs on multiple Motif-Variant datasets with distinct biases.}
	\label{fig:LC}
\end{figure*}

\begin{figure*}[h]
	\centering
  	\subfigure[Output feature visualization of Local Extremum GNN.]{
		\begin{minipage}{0.4\textwidth} 
			\includegraphics[width=\textwidth]{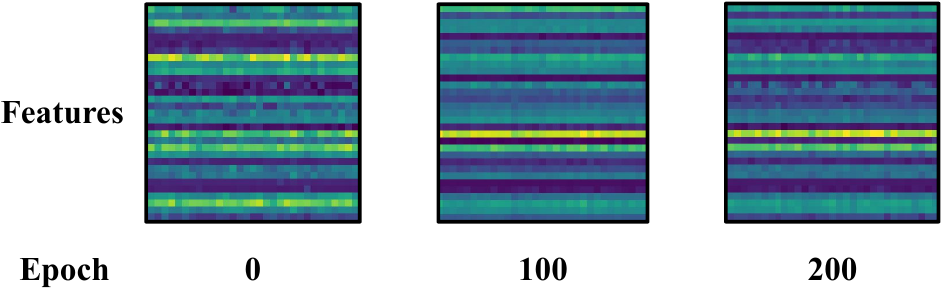} \\
                \label{fig:vis1a}
		\end{minipage}
    }\hspace{5mm}
	\subfigure[Output feature visualization of DCSGL.]{
		\begin{minipage}{0.4\textwidth} 
			\includegraphics[width=\textwidth]{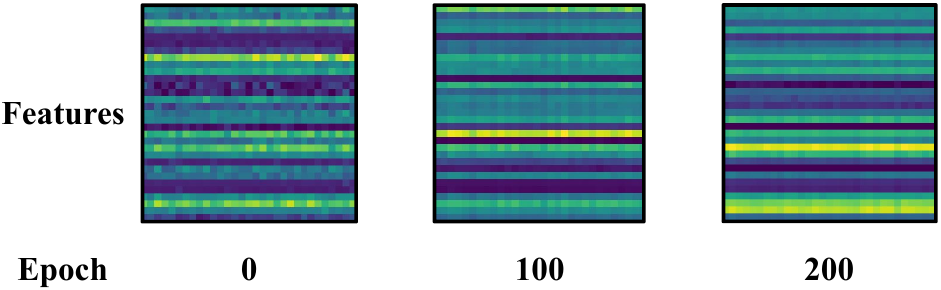} \\
                \label{fig:vis1b}
		\end{minipage}
	}
	\subfigure[t-SNE visualization results of Local Extremum GNN.]{
		\begin{minipage}{0.4\textwidth} 
			\includegraphics[width=\textwidth]{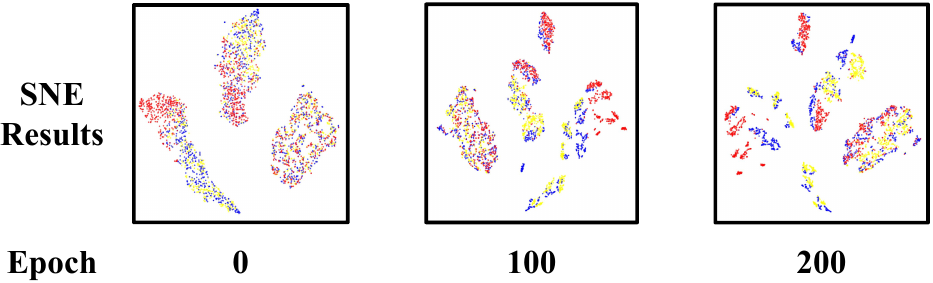} \\
                \label{fig:vis2a}
		\end{minipage}
    }\hspace{5mm}
	\subfigure[t-SNE visualization results of DCSGL.]{
		\begin{minipage}{0.4\textwidth} 
			\includegraphics[width=\textwidth]{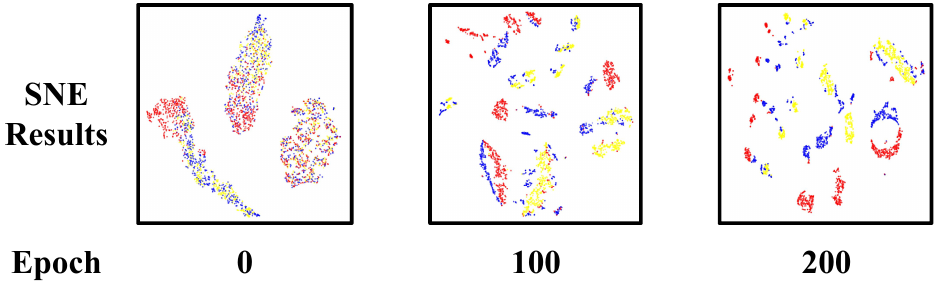} \\
                \label{fig:vis2b}
		\end{minipage}
	}

	\caption{Visualization of model outputs.}
	\label{fig:vis}
\end{figure*}

\begin{figure*}[h]
	\centering
	\subfigure[Hyperparameter experiments on Spurious-Motif (ID) and Motif-Variant (ID).]{
		\begin{minipage}{0.45\textwidth} 
			\includegraphics[width=\textwidth]{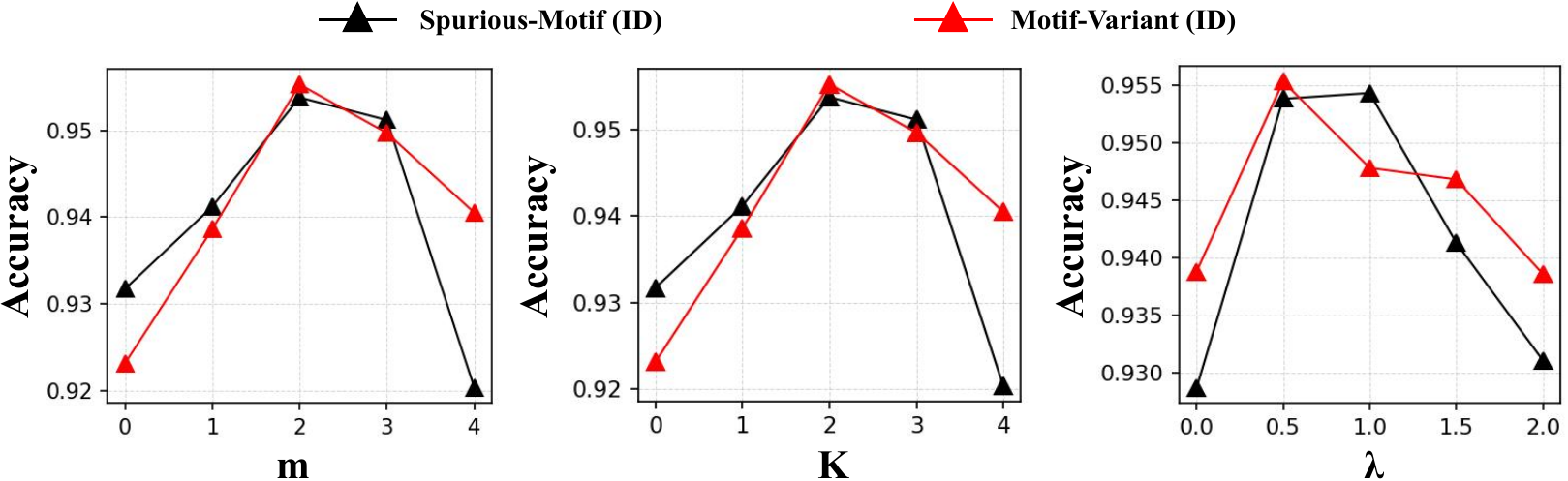} \\
                \label{fig:hypa}
		\end{minipage}\hspace{5mm}
    }
	\subfigure[Hyperparameter experiments on Spurious-Motif (OOD) and Motif-Variant (OOD).]{
		\begin{minipage}{0.45\textwidth} 
			\includegraphics[width=\textwidth]{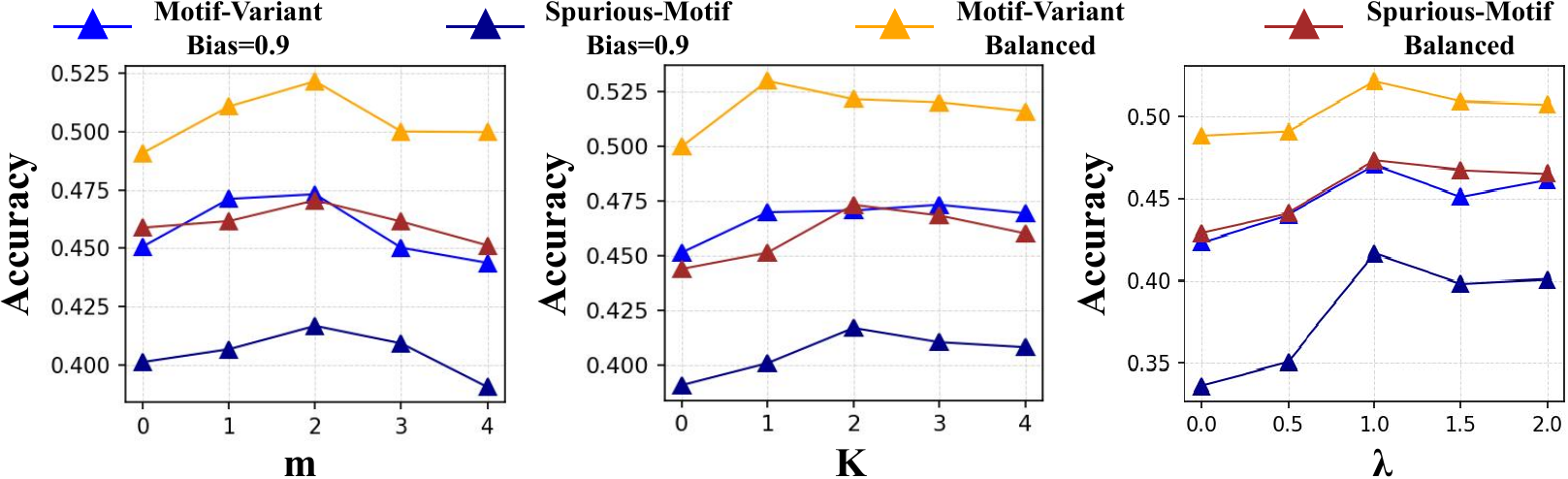} \\
                \label{fig:hypb}
		\end{minipage}
   }
	\subfigure[Hyperparameter experiments on Graph-SST2 (ID), Graph-twitter (ID) and Graph-SST5 (ID).]{
		\begin{minipage}{0.45\textwidth} 
			\includegraphics[width=\textwidth]{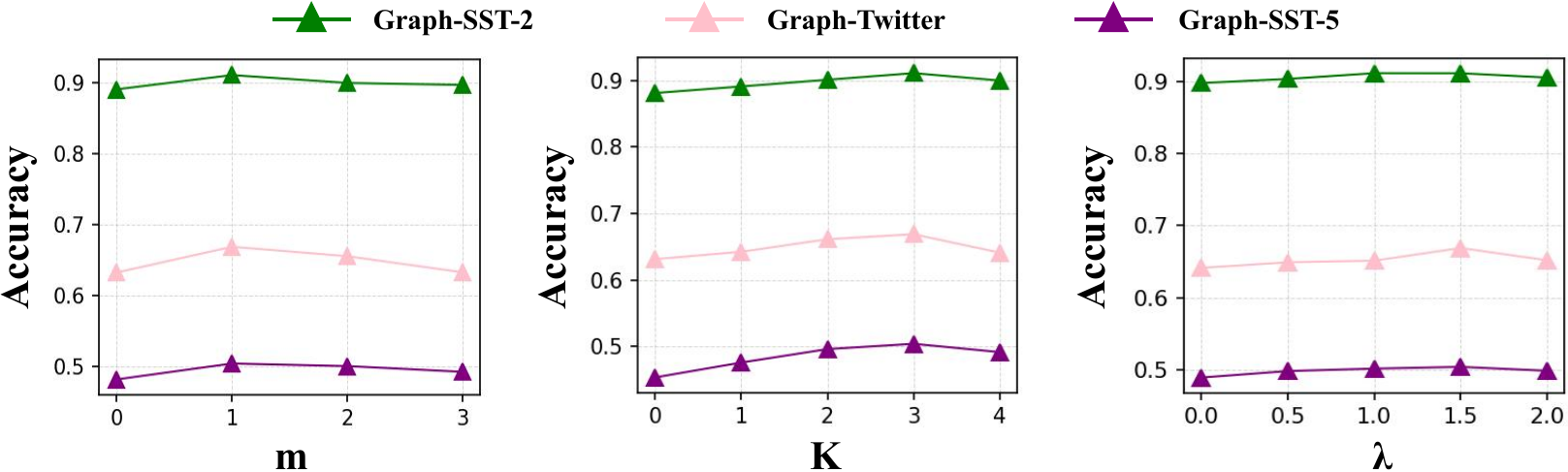} \\
                \label{fig:hypc}
		\end{minipage}\hspace{5mm}
   }
	\subfigure[Hyperparameter experiments on Graph-SST2 (OOD), Graph-twitter (OOD) and Graph-SST5 (OOD).]{
		\begin{minipage}{0.45\textwidth} 
			\includegraphics[width=\textwidth]{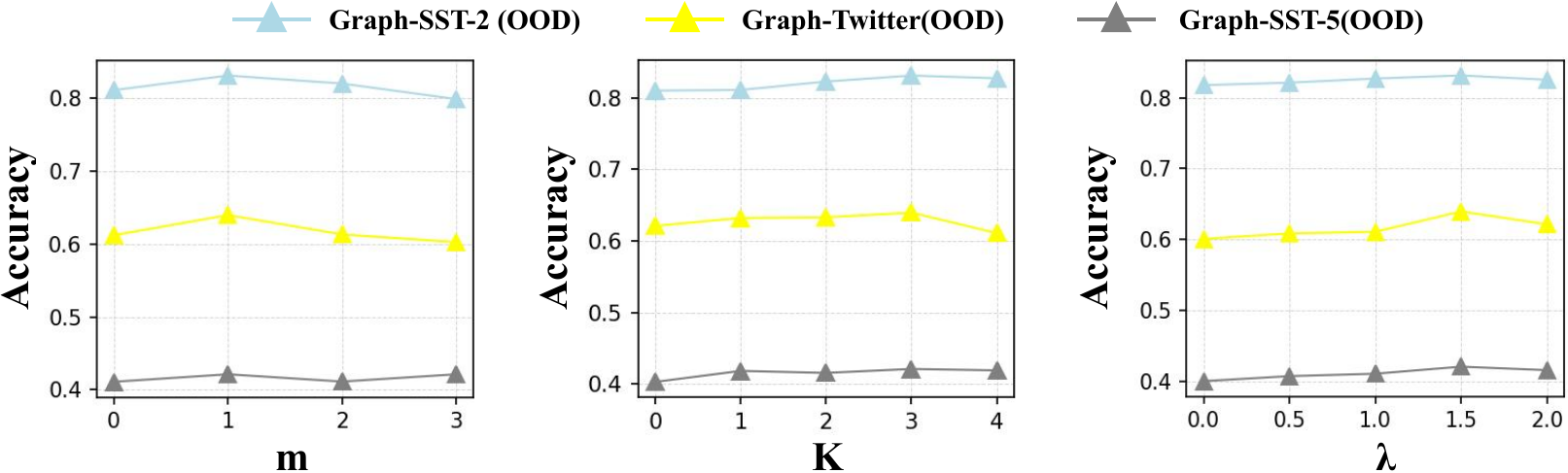} \\
                \label{fig:hypd}
		\end{minipage}
	}

	\caption{Hyperparameter Experiments.}
	\label{fig:hyp}
\end{figure*}

For the diminutive causal structure, our focus lies in the shape of junctions between distinct motifs. Graphs rich in topological information are often composed of numerous motifs with specific shapes. The junctions between these motifs typically exhibit shapes distinct from the motifs themselves. Identifying these junctions can contribute to a clearer analysis of the topological structure of the graph. However, recognizing junctions alone is insufficient for providing direct assistance in predictions. Thus, we consider the principles governing the composition of junctions as a diminutive causal structure.

To provide a clearer explanation, practical examples of motif junctions are presented in Figure \ref{fig:jun}. Subsequently, we proceed with the construction of $\mathcal{M}(\cdot)$. We employ a graph search algorithm to identify the junctions between the motifs. These discovered junctions are then labeled. Since the data in the training dataset remains constant throughout the training process, labeling needs to be performed only once, and subsequent repetitions of this process are unnecessary.

We built $\mathcal{M}(\cdot)$ to judge whether a specific node is a junction. Clearly, labeled nodes are junctions, while other nodes are not. Therefore, our $\phi(\cdot)$ is designed to output the features of nodes labeled as junctions. Thus, when the input corresponds to a labeled node, $\mathcal{M}(\cdot)$ outputs a probability of 1.

For $\mathcal{M}^{\gamma}(\cdot)$ and $\phi^{\gamma}(\cdot)$, we randomly replace 50\% of the features of the nodes output by $\phi(\cdot)$ with the features of other nodes. In this scenario, the corresponding output probability of $\mathcal{M}^{\gamma}(\cdot)$ is 0.5. We repeat the random replacement for different $\phi^{\gamma}(\cdot)$.

\subsubsection{Settings}
We compare our method with the backbone network named Local Extremum GNN and various causality-enhanced methods that can be divided into two groups: 1) the interpretable baselines, including GAT, Top-k Pool and GSN, 2) the robust learning baselines, including Group DRO, IRM, V-REx and DIR. Furthermore, hierarchical pooling methods such as MTPool \cite{DBLP:journals/nn/DuanXWHRXSW22}, may also positively impact causal information extraction. However, we illustrate this with the example of Top-k Pool and do not delve into a detailed one-to-one comparison. These approaches differ significantly from our method in that they place greater emphasis on internal modifications within the model, aiming to achieve causal modeling capabilities through architectural innovations. In contrast, our method prioritizes the introduction of external causal information. For a fair comparison, we follow the experimental principles of \cite{DBLP:conf/iclr/WuWZ0C22} and adopt the same training setting for all models. For each task, we report the mean performance ± standard deviation over five runs.

Moreover, to substantiate that the diminutive causal structure we have adopted does not inherently contribute directly to predictions, we have introduced additional models named DCS-Only-D and DCS-Only-L. In DCS-Only-D, junction marks (assigned as 1 for junction nodes and 0 for other nodes) are treated as extra node classification task labels. Meanwhile, DCS-Only-L incorporates the marking of junctions as extra graph classification training labels, providing information on the number of junctions within the graph.

To further justify the structure of our proposed method, we have introduced the DCSGL-T model. The DCSGL-T model accepts input from the last layer instead of the $m$-th layer. Additionally, we developed the DCSGL-A model to conduct further ablation studies, wherein the enhancement with interchange interventions is deliberately replaced with negative samples from random select node features.

For hyperparameters, we specify $m=2$, $\lambda=1$, and $K=3$. The learning rate is set to 0.001. Our model undergoes training for 200 epochs, and early stopping is implemented if there is no observed improvement in performance on the validation set for five consecutive epochs. The model exhibiting the best validation performance is selected as the final result. The maximum allowable number of training epochs is established at 400 for all datasets. A training batch size of 32 is employed. The backbone of our GNN architecture is the Local Extremum GNN \cite{DBLP:conf/aaai/RanjanST20}, featuring 4 layers and a hidden dimension of 32. Global mean pooling is adopted as the pooling method.

All our experiments are performed on a workstation equipped with two Quadro RTX 5000 GPUs (16 GB), an Intel Xeon E5-1650 CPU, 128GB RAM and Ubuntu 20.04 operating system. We employ the Adam optimizer for optimization. We set the maximum training epochs to 400 for all tasks. For backpropagation, we optimize Graph-SST2 and Graph-Twitter using stochastic gradient descent (SGD), while utilizing gradient descent (GD) on Spurious-Motif and Motif-Variant.

\subsubsection{Results}

The results are reported in Table \ref{tab:sm} and \ref{tab:smv}. From the results, we observe that our proposed DCSGL method consistently outperforms various baseline models across multiple datasets, demonstrating the efficacy of our approach. Additionally, it is noteworthy that the performance of DCS-Only-D and DCS-Only-L, two baseline models incorporating the diminutive causal structure, does not exhibit significant improvement compared to the fundamental backbone GNN. This suggests that the adopted diminutive causal structure is not directly task-relevant. Furthermore, the performance of DCSGL significantly surpasses the average performance of DCS-Only-D and DCS-Only-L, highlighting the utility of our introduced diminutive causal structure. This empirical evidence substantiates the effectiveness of our approach. Additionally, the ablative models, DCSGL-T and DCSGL-A, exhibit weaker performance compared to DCSGL, emphasizing the necessity of the proposed modules. This underscores the importance and effectiveness of our introduced methodology.

\subsection{Experiments with Diminutive Causal Structure from Linguist Domain}

In this section, we anchor our validation efforts on a selection of real-world datasets. The semantic graphs derived from textual information are chosen for their origin in textual knowledge, facilitating a more straightforward analysis of their internal logical structures. Additionally, datasets within the field of textual graph analysis are widely employed for the validation of graph representation learning methods \cite{DBLP:journals/tnn/WuPCLZY21}. Consequently, we employ these datasets, along with diminutive causal structures extracted from the field of linguistics, for the experiments conducted in this section.

\subsubsection{Dataset}
The datasets employed in this experiment include multiple real-world datasets, including Graph-SST2, Graph-SST5, and Graph-Twitter \cite{yuan2020explainability}, along with their OOD versions. The OOD versions were generated following \cite{DBLP:conf/iclr/WuWZ0C22}. The details of the datasets are demonstrated within Table \ref{tab:datasets2}.

\subsubsection{Model $\mathcal{M}(\cdot)$ Construction}
\label{sec:mds2}
In this section, we construct diminutive causal structures based on relevant knowledge about the conjunction `but.' A brief description of the corresponding model construction for these causal structures has been provided in the preceding text. Essentially, we employ the semantic contrastive relationship represented by the conjunction `but' to facilitate training. An illustrative example of the role of the conjunction 'but' in the textual graph is presented in Figure \ref{fig:example}.

For the model $\mathcal{M}(\cdot)$, we generate predictions regarding the existence of the conjunction `but' and its associated contrastive structure. Clearly, we can directly assess the presence of this conjunction based on the corresponding textual data in the graph. If it is present, the output probability is set to 1. Concerning $\phi(\cdot)$, we use the statement containing `but' as input. To avoid introducing additional interference, such as the presence of other types of contrastive conjunctions, we refrain from using statements without `but' during training. This results in a problem of missing negative samples.

To address this issue, we employ the interchange intervention method mentioned in Section \ref{sec:intercinterv}. Specifically, we apply $\phi^{\gamma}(\cdot)$ to delete the conjunction `but' and its subsequent content, thereby removing the contrastive relationship. Simultaneously, we have the corresponding model $\mathcal{M}^{\gamma}(\cdot)$ output a judgment result of 0, indicating the absence of the conjunction `but.' Through this approach, we supplement the negative samples. We delete different amounts of contents for different $\phi^{\gamma}(\cdot)$, and adjust the output of $\mathcal{M}^{\gamma}(\cdot)$ according to the percentage of content removed.

\subsubsection{Settings}
We maintain consistency with the baselines and most settings from the previous experiment, with the exception of the backbone GNN and hyperparameters. Specifically, we adopt ARMA \cite{DBLP:conf/aaai/0001RFHLRG19} as the GNN backbone, with three layers and a hidden dimension of 128. We set $m=2$, $K=3$, $\lambda=1.5$, and the learning rate to 0.0002.

\subsubsection{Results}
The results are demonstrated in Table \ref{tab:graph}. From the results, it is evident that DCSGL achieved optimal performance across all datasets. Additionally, when compared to the average values of DCS-Only-D and DCS-Only-L, DCSGL demonstrated a significant improvement in performance. These experimental observations substantiate the efficacy of our approach.

\subsection{In-Depth Study}

To gain a more profound understanding of the characteristics of our approach, we conducted a series of analytical experiments.

\subsubsection{Generalizability Analysis}



To analyze the generalization capability of our method, we conducted additional experiments on the Motif-Variant dataset by replacing the backbone network. Compared to the original backbone network, we introduced two new backbones: FAGCN and EGC. We separately tested the performance of the models using only the backbone network and the models incorporating the DCSGL method.

The experimental results are presented in Table \ref{tab:multibaseline}. It can be observed that DCSGL achieved certain improvements on the two newly introduced baselines. On datasets with both confounding factors and those without confounding factors, DCSGL consistently enhanced the algorithm's performance. These experimental results strongly validate the generalization capability of our proposed method. We attribute this generalization to the structural design of DCSGL, which does not rely on specific network architectures but can inject subtle causal structures into different networks.

\subsubsection{Computational Complexity Analysis}

For the purpose of computational complexity analysis, we will undertake additional investigation and experimentation. The computational complexity of DCSGL is intricately linked to the volume of graph data involved in the computation and the number of GNN layers engaged. During the computation of logic learning loss, a subset of graph data was selected, and a portion of network layers was extracted to perform forward propagation on the GNN. Therefore the computational complexity associated with this part is similar to that of graph neural networks, scaling linearly with the number of edges in the selected graph data, i.e., $\mathcal{O}(|\mathcal{E}| \times F)$, where $\mathcal{E}$ denotes the selected edge set, $F$ denotes the computational complexity induced by dimensionality. Based on Equation \ref{eq:Lc}, the subsequent addition of KL divergence calculation, pooling and MLP operations maintains the computational complexity at $\mathcal{O}(|\mathcal{E}| \times F)$. Consequently, the additional computational complexity of DCSGL is less than the computational complexity of the backbone GNN, and proportional to $|\mathcal{E}|$. We have conducted further analysis through experimental validation.

We separately recorded the computation time and GPU memory consumption for a single epoch of the standalone backbone network, DIR, and our proposed method on the Spurious-Motif dataset and Graph SST2 dataset, as depicted in Table \ref{tab:tc}. It is observed that, although our method incurs an additional computational time compared to the standalone backbone network, this overhead is less than that of DIR. Furthermore, our method achieves superior performance compared to DIR, indicating that such computational costs are deemed acceptable.

\subsubsection{Training Procedure Evaluation}

To gain a comprehensive understanding of the dynamic evolution of our model during the training process, we conducted meticulous monitoring, encompassing both accuracy and $\mathcal{L}_{a}$ loss, across multiple Motif-Variant datasets characterized by varying bias coefficients. Additionally, we performed a similar $\mathcal{L}_{a}$ loss calculation for the Local Extremum GNN, acting as the backbone GNN, akin to the DCSGL method, albeit without utilizing backpropagation for model updates.

The results presented in Figure \ref{fig:LC} distinctly illustrate a rapid and substantial reduction in the $\mathcal{L}_{a}$ loss for DCSGL over the entire course of 200 epochs, maintaining a consistently low level. Simultaneously, the model's performance continues to improve even after multiple epochs. These findings unequivocally indicate that DCSGL not only possesses the capability to effectively diminish causal structure loss during the learning process but also sustains a continuous enhancement in its performance over multiple epochs.

Notably, we observed that datasets with different bias coefficients have similar values of $\mathcal{L}_{a}$ loss, suggesting that the encoded diminutive causal structure by DCSGL is not intricately intertwined with the complexity of downstream tasks. Consequently, DCSGL demonstrates robust learning efficiency in the face of dataset variations and varying task contexts. 

We observe distinct behavior in the baseline model, Local Extremum GNN. Despite a gradual decrease in the $\mathcal{L}_{a}$ loss during training, it does not exhibit as pronounced a reduction as the DCSGL method. This phenomenon aligns with the experimental findings presented in Figure \ref{fig:mtv1}, indicating that, throughout training, a general GNN model can to some extent acquire knowledge about the diminutive causal structure. However, the model falls short of fully capturing such a causal structure. Additionally, the knowledge of the diminutive causal structure acquired by Local Extremum GNN may be prone to forgetting as evidenced by the $\mathcal{L}_{a}$ loss resurging in later stages, as illustrated in the graph.

Furthermore, we could also see a decline in the performance of Local Extremum GNN on certain datasets during the later stages of training. Consequently, we can conclude that, due to the absence of explicit support for a well-defined diminutive causal structure akin to DCSGL, Local Extremum GNN struggles to sustain a more continuous and stable learning process. This limitation becomes particularly evident in datasets containing confounding factors.

Our observations lead us to conclude that DCSGL, through the effective integration of diminutive causal structures, enables the model to consistently maintain a correct causal framework at specific locations. This stability aids the model in preserving a consistent understanding of the desired causal relationships throughout the learning process, thereby facilitating continuous and effective learning. The adaptive learning strategy allows DCSGL to better accommodate dynamic changes in tasks and diverse biases in datasets, showcasing its substantial potential in complex and diverse real-world applications. In practical scenarios, leveraging the advantages of incorporating diminutive causal structures lies in the ability to make optimal use of readily available domain knowledge, ensuring outstanding performance across various tasks within the domain.

\subsubsection{Visualized Evaluation}

Figure \ref{fig:vis} illustrates the visualized output features of our model, presenting two modes of data visualization. Specifically, Figure \ref{fig:vis1a} and Figure \ref{fig:vis1b} provide visualization of the output features. In each block of these two figures, each vertical line represents the model's representation of the graph sample. All showcased sample features belong to the same category, and thus, the horizontal lines in the figure signify the model's representation at the same position for different samples within the same category. In essence, smoother horizontal lines in the figure indicate a more consistent information representation between samples, while noise on the horizontal lines suggests inconsistent features between samples.

In comparison to the backbone network Local Extremum GNN, DCSGL generates more consistent representations at the 100th and 200th epochs, highlighting its significant superiority in knowledge learning. Furthermore, there are notable changes in DCSGL's representations between the 100th and 200th epochs, indicating that the model progressively optimizes its feature representation over continuous learning. This allows the model to continually acquire new knowledge while ensuring a stable representation of certain causal knowledge that remains unaffected by new information.

Figure \ref{fig:vis2a} and \ref{fig:vis2b} demonstrates the feature representations of all samples after dimensionality reduction using the t-SNE \cite{van2008visualizing} algorithm. Different colors in the figure represent different categories. It is evident from the figure that DCSGL can indeed learn more refined and accurate features. This outstanding performance is attributed to the injection of minute causal relationships into our model, enabling the model to continuously optimize its feature representation based on accurately determined knowledge, thereby ensuring the model's accuracy. The introduction of these subtle causal relationships enhances the model's flexibility in adapting to and understanding complex data structures, providing a more reliable foundation for its application in various task scenarios. In conclusion, these results thoroughly validate the superiority of our model and its robust capability in learning and representing causal relationships.

\subsubsection{Hyperparamter Analysis}

In Figure \ref{fig:hyp}, the experimental results regarding $m$ and $K$ further confirm that extracting representations in the intermediate layers for causal knowledge learning, along with the simultaneous application of multiple exchange interventions, significantly enhances the model's performance in causal relationship learning. It underscores the effectiveness of our approach in strengthening the model's understanding of causal relationships. It is noteworthy that by adjusting the $\lambda$ parameter, we can effectively balance the relationship between causal learning and dataset training, providing better control and adjustment for the overall performance of the model.

These experimental results not only emphasize the necessity of the proposed structure but also deepen our understanding of the model's performance. Specifically, our structural design, considering the joint utilization of intermediate layer representations and multiple exchange interventions, plays a crucial role in causal knowledge learning. Through sensitivity analysis of the $\lambda$ parameter, we further confirm its importance in balancing causal learning and dataset training.

Furthermore, we observe a relatively consistent performance variation of the model across different datasets within the same domain. This suggests that despite dataset differences, the consistent introduction of diminutive causal structures results in a uniform trend in model performance. This further highlights the generality and robustness of our proposed method, enabling it to consistently outperform on diverse datasets.

\section{Conclusions}

Drawing on empirical observations from motivating experiments, it becomes clear that integrating diminutive causal structures into graph representation learning significantly enhances model performance. In response, we present DCSGL, a methodology explicitly crafted to guide the model in learning these specific causal structures. DCSGL accomplishes this objective by leveraging insights gained from models constructed based on diminutive causal structures. The learning process is further refined through interactive interventions. The validity and effectiveness of our approach are substantiated through rigorous theoretical analyses. Additionally, empirical comparisons underscore the substantial performance enhancements achieved by DCSGL.

\section*{Acknowledgements}
The authors would like to thank the editors and reviewers for their valuable comments. This work is supported by the National Funding Program for Postdoctoral Researchers, Grant No. GZC20232812, the Fundamental Research Program, Grant No. JCKY2022130C020, the CAS Project for Young Scientists in Basic Research, Grant No. YSBR-040. 

\bibliographystyle{elsarticle-num} 
\bibliography{references}


\end{document}